\documentclass[preprint,12pt]{elsarticle}

\usepackage{amssymb}
\usepackage{amsmath}

\usepackage{amsmath}
\usepackage{amssymb}
\usepackage{mathtools}
\usepackage{amsthm}

\usepackage{enumitem}

\theoremstyle{plain}
\newtheorem{theorem}{Theorem}[section]
\newtheorem{proposition}[theorem]{Proposition}
\newtheorem{lemma}[theorem]{Lemma}
\newtheorem{corollary}[theorem]{Corollary}
\theoremstyle{definition}
\newtheorem{definition}[theorem]{Definition}

\theoremstyle{remark}
\newtheorem{remark}[theorem]{Remark}

\usepackage[textsize=tiny]{todonotes}

\usepackage{amsmath,amsfonts,bm}

\def\eqref#1{(\ref{#1})}

\def\1{\bm{1}}

\def\vone{{\bm{1}}}

\def\vb{{\bm{b}}}

\def\ve{{\bm{e}}}

\def\vh{{\bm{h}}}

\def\vu{{\bm{u}}}
\def\vv{{\bm{v}}}
\def\vw{{\bm{w}}}
\def\vx{{\bm{x}}}

\def\vz{{\bm{z}}}

\def\vone{{\bm{1}}}

\def\mA{{\bm{A}}}
\def\mB{{\bm{B}}}

\def\mD{{\bm{D}}}

\def\mG{{\bm{G}}}
\def\mH{{\bm{H}}}
\def\mI{{\bm{I}}}

\def\mQ{{\bm{Q}}}

\def\mW{{\bm{W}}}

\def\mY{{\bm{Y}}}
\def\mZ{{\bm{Z}}}

\DeclareMathAlphabet{\mathsfit}{\encodingdefault}{\sfdefault}{m}{sl}
\SetMathAlphabet{\mathsfit}{bold}{\encodingdefault}{\sfdefault}{bx}{n}

\def\gM{{\mathcal{M}}}

\def\sR{{\mathbb{R}}}

\usepackage{thm-restate}

\usepackage[utf8]{inputenc} 
\usepackage[T1]{fontenc}    
\usepackage{hyperref}       
\usepackage{url}            
\usepackage{booktabs}       
\usepackage{amsfonts}       
\usepackage{nicefrac}       
\usepackage{microtype}      
\usepackage{xcolor}         

\usepackage{mathtools}

\begin{document}

\begin{frontmatter}



\title{
Learning to Control the Smoothness of Graph Convolutional Network Features
}


\author{Shih-Hsin Wang$^{*a}$, Justin Baker$^{*a}$\footnote{Shih-Hsin Wang and Justin Baker are co-first authors of this paper.},
Cory Hauck$^{b}$, Bao Wang$^{a}$} 
\affiliation{organization={Departmenr of Mathematics and Scientific Computing and Imaging Institute, University of Utah},
            addressline={72 Central Campus Dr}, 
            city={Salt Lake City},
            postcode={84102}, 
            state={Utah},
            country={USA}}

\affiliation{organization={Oak Ridge National Laboratory and University of Tennessee},
            addressline={5200, 1 Bethel Valley Rd, Oak Ridge}, 
            city={Oak Ridge},
            postcode={37830}, 
            state={Tennessee},
            country={USA}}

\begin{abstract}
The pioneering work of Oono and Suzuki [ICLR, 2020] and Cai and Wang [arXiv:2006.13318] initializes the analysis of the smoothness of graph convolutional network (GCN) features. Their results reveal an intricate empirical correlation between node classification accuracy and the ratio of smooth to non-smooth feature components. However, the optimal ratio that favors node classification is unknown, and the non-smooth features of deep GCN with ReLU or leaky ReLU activation function diminish. In this paper, we propose a new strategy to let GCN learn node features with a desired smoothness -- adapting to data and tasks -- to enhance node classification. Our approach has three key steps: (1) We establish a geometric relationship between the input and output of ReLU or leaky ReLU. (2) Building on our geometric insights, we augment the message-passing process of graph convolutional layers (GCLs) with a learnable term to modulate the smoothness of node features with computational efficiency. (3) We investigate the achievable ratio between smooth and non-smooth feature components for GCNs with the augmented message-passing scheme. Our extensive numerical results show that the augmented message-passing schemes significantly improve node classification for GCN and some related models.
\end{abstract}


\end{frontmatter}



\section{Introduction}\label{sec:intro}
Let $G=(V,E)$ be an undirected graph with $V=\{v_i\}_{i=1}^n$ and $E$ being the set of $n$ nodes and the set of edges, respectively. 
Let $\mA\in\sR^{n\times n}$ be the adjacency matrix of the graph with $A_{ij}=\vone_{(i,j)\in E}$, where $\vone$ is the indicator function. Furthermore, let $\mG$ be the augmented normalized adjacency matrix, given as follows:
\begin{equation}\label{eq:G}
\begin{aligned}
\mG \coloneqq (\mD+\mI)^{-\frac{1}{2}}(\mI+\mA)(\mD+\mI)^{-\frac{1}{2}}=\Tilde{\mD}^{-\frac{1}{2}}\Tilde{\mA}\Tilde{\mD}^{-\frac{1}{2}},
\end{aligned}
\end{equation}
where $\mI\in\sR^{n\times n}$ is the identity matrix, $\mD\in\sR^{n\times n}$ is the diagonal degree matrix with ${  D_{ii}=\sum_{j=1}^nA_{ij}}$, and $\Tilde{\mA}:=\mA+\mI$ and $\Tilde{\mD}:=\mD+\mI$. Starting from the initial node features $\mH^0:=$ $[(\vh_1^0)^{\top},(\vh_2^0)^{\top},\ldots,(\vh_n^0)^\top]^\top\in\sR^{d\times n}$ with $\vh_i^0\in\sR^d$ being the $i^{th}$ node feature vector, the graph convolutional network (GCN) \cite{kipf2017semisupervised} learns node representations using the following graph convolutional layer (GCL) transformation:
\begin{equation}
\label{eq:GCN}
\begin{aligned}
\mH^{l}=\sigma(\mW^{l}\mH^{l-1}\mG),
\end{aligned}
\end{equation}
where $\sigma$ is the activation function, e.g. ReLU \cite{nair2010rectified}, and $\mW^l\in \sR^{d\times d}$ is learnable. Notice that GCL smooths feature vectors of the neighboring nodes via a weighted averaging procedure. The smoothness of node features has been argued to help node classification; see e.g. \cite{li2018deeper,pmlr-v97-wu19e,chen2020measuring}, resonating with the idea of classical energy-based semi-supervised learning approaches \cite{zhu2003semi,zhou2003learning}. Accurate node classification requires a balance between smooth and non-smooth components of GCN node features~\cite{oono2019graph}. Besides graph convolutional networks (GCNs) stacking GCLs, many other graph neural networks (GNNs) have been developed using different mechanisms, including spectral methods \cite{defferrard2016convolutional}, spatial methods \cite{gilmer2017,velivckovic2018graph,wang2024rethinking}, sampling methods \cite{hamilton2017inductive,ying2018graph}, and the attention mechanism \cite{velivckovic2018graph}.  Many other GNN
models can be found in recent surveys or monographs; see, e.g. \cite{hamilton2017representation,battaglia2018relational,wu2020comprehensive,zhou2020graph,hamilton2020graph}.

Deep neural networks usually significantly outperform shallow counterparts, and a remarkable example is convolutional neural networks \cite{krizhevsky2017imagenet,He_2016_CVPR}. However, such a property does not hold for GCNs; it has been observed that deep GCNs that consist of multiple GCLs tend to perform significantly worse than shallow models \cite{chen2020measuring}. In particular, the node feature vectors learned by deep GCNs tend to be identical to each other over each connected component of the graph; this phenomenon is referred to as {\bf\emph{over-smoothing}}
\cite{li2018deeper,nt2019revisiting,oono2019graph,cai2020note,chen2020measuring,wu2022non}, which not only occurs for GCN but also for many other GNNs, e.g., GraphSage \cite{hamilton2017inductive} and MPNN \cite{gilmer2017}. Intuitively, each GCL smooths neighboring node features, which benefits node classification
\cite{li2018deeper,pmlr-v97-wu19e,chen2020measuring}. However, stacking these smoothing layers will inevitably homogenize node features. Algorithms have been developed to alleviate the over-smoothing issue of GNNs, including decoupling prediction and message passing \cite{gasteiger2018combining}, skip connection and batch normalization \cite{kawamoto2018mean,chen2018supervised,chen2020simple}, 
graph sparsification \cite{Rong2020DropEdge:}, jumping knowledge \cite{xu2018representation}, scattering transform \cite{min2020scattering}, PairNorm \cite{zhao2019pairnorm}, implicit layers using continuous-in-depth layers~\cite{pmlr-v139-chamberlain21a,thorpe2022grand} or fixed-point network~\cite{NEURIPS2020_8b5c8441,pmlr-v202-baker23a,baker2024monotone}, and controlling the Dirichlet energy of node features \cite{zhou2021dirichlet}. 

From a theoretical perspective, it is proved that deep GCNs using ReLU or leaky ReLU activation function tend to learn homogeneous node features across graph nodes of each connected component
\cite{oono2019graph,cai2020note}. In particular, Oono and Suzuki~\cite{oono2019graph} show that the distance of node features to the eigenspace $\gM$ -- corresponding to the largest eigenvalue 1 of the message-passing matrix $\mG$ in \eqref{eq:G} -- goes to zero when the depth of GCN with the ReLU activation function goes to infinity. Meanwhile, Oono and Suzuki \cite{oono2019graph} empirically study the intricate correlation between node classification accuracy and the ratio between the smooth and non-smooth components of GCN node features, i.e., projections of node features onto the eigenspace $\gM$ and its orthogonal complement $\gM^\perp$, respectively. The empirical results of \cite{oono2019graph} indicate that both smooth and non-smooth components of node features are crucial for accurate node classification, while the ratio between smooth and non-smooth components to achieve optimal accuracy is unknown and task-dependent. Furthermore, Cai and Wang \cite{cai2020note} prove that the Dirichlet energy -- another smoothness measure for node features -- goes to zero when the depth of GCN with ReLU or leaky ReLU activation function goes to infinity.

A crucial step in the proofs of the results in \cite{oono2019graph,cai2020note} is that both ReLU and leaky ReLU activation functions reduce the distance of node feature vectors to $\gM$ and their Dirichlet energy. However, Cai and Wang \cite{cai2020note} point out that over-smoothing -- characterized by the distance of features to eigenspace $\gM$ or the Dirichlet energy -- is a misnomer; the real smoothness of a graph signal should be characterized by a {\bf\emph{normalized smoothness}}, e.g., normalizing the Dirichlet energy by the magnitude of the features. Indeed, the ratio between smooth and non-smooth components of node features -- as studied by Oono and Suzuki \cite{oono2019graph} -- is closely related to the normalized smoothness. 
Nevertheless, analyzing the normalized smoothness of node features learned by GCN with ReLU or leaky ReLU activation functions remains an open problem \cite{cai2020note}. Moreover, it is interesting to ask if analyzing the normalized smoothness can result in any new understanding of GCN features and algorithms to improve the node classification accuracy for GCN.

\subsection{Our Contribution}
We aim to (1) establish a new geometric understanding of how GCL smooths GCN features and (2) develop an efficient algorithm to let GCN and related models learn node features with a desired ratio between their smooth and non-smooth components or normalized smoothness \cite{cai2020note} --  to improve node classification accuracy. We summarize our main contributions towards achieving our goal as follows:
\begin{itemize}
\item We prove that there is a high-dimensional sphere underlying the input and output vectors of ReLU or leaky ReLU activation functions. This geometric characterization not only implies theories in \cite{oono2019graph,cai2020note} but also informs that adjusting the projection of input vectors onto the eigenspace $\gM$ can alter the smoothness of the output vectors. See Section~\ref{sec:geometry} for details.

\item We show that both ReLU and leaky ReLU activation functions reduce the distance of node features to the eigenspace $\gM$, i.e., ReLU and leaky ReLU smooth their input vectors without considering their magnitude. In contrast, we show that when taking the magnitude into account, ReLU and leaky ReLU activation functions can increase, decrease, or preserve the normalized smoothness of each dimension of their input vectors; see Sections~\ref{sec:geometry} and \ref{sec:control}.

\item Inspired by our established geometric relationship between the input and output of ReLU or leaky ReLU, we study how adjusting the projection of input onto the eigenspace $\gM$ affects both normalized and unnormalized smoothness of the output vectors. We show that the distance of the output vectors to the to eigenspace $\gM$ is always no greater than that of the original input -- no matter how we adjust the input by changing its projection onto $\gM$. In contrast, adjusting the projection of input vectors onto $\gM$ can effectively change the normalized smoothness of output to any desired value; see details in Section~\ref{sec:control}.

\item 
Based on our theory, we propose a computationally efficient smoothness control term (SCT) to let GCN and related models learn node features with a desired (normalized) smoothness to improve node classification. We comprehensively validate the benefits of our proposed SCT in improving node classification -- for both homophilic and heterophilic graphs -- using a few of the most representative GCN-style models. See Sections~\ref{sec:SCT} and \ref{sec:experiments} for details.
\end{itemize}

As far as we know, our work is the first thorough study of how ReLU and leaky ReLU activation functions affect the smoothness of node features both with and without considering their magnitude. 

\subsection{Additional Related Works}
Controlling the smoothness of node features to improve the performance of GCNs is another line of related work. For instance, Zhao and Akoglu \cite{zhao2019pairnorm} design a normalization layer to prevent node features from becoming too similar to each other, and Zhou et al. \cite{zhou2021dirichlet} constrain the Dirichlet energy to control the smoothness of node features without considering the effects of nonlinear activation functions. While there has been effort in understanding and alleviating the over-smoothing of GCNs and controlling the smoothness of node features, there is a shortage of theoretical examination of how activation functions affect the smoothness of node features, specifically accounting for the magnitude of features.

\subsection{Notation and Organization}

\noindent{\bf Notation.}
We denote the $\ell_2$-norm of a vector $\vu$ as $\|\vu\|$. For vectors $\vu$ and $\vv$, we use $\langle \vu,\vv\rangle$, $\vu\odot \vv$, and $\vu\otimes \vv$ to denote their inner, Hadamard, and Kronecker product, respectively. For a matrix $\mA$, we denote its $(i,j)^{th}$ entry, transpose, and inverse as $A_{ij}$, $\mA^\top$, and $\mA^{-1}$, respectively. We denote the trace of $\mA\in\sR^{n\times n}$ as ${\rm Trace}(\mA)=\sum_{i=1}^nA_{ii}$.  For two matrices $\mA$ and $\mB$, we denote the Frobenius inner product as $\langle \mA,\mB\rangle_F:={\rm Trace}(\mA\mB^\top)$ and the Frobenius norm of $\mA$ as  $\|\mA\|_F:=\sqrt{\langle\mA,\mA\rangle}$.

\medskip
\noindent{\bf Organization.}
We provide preliminaries and a review of some existing related results
in Section~\ref{subsec:orthogonal-decomposition}. In Section~\ref{sec:geometry}, we establish a geometric characterization of how ReLU and leaky ReLU activation functions affect the smoothness of their input vectors. We study the smoothness of each dimension of node features and take their magnitude into account in Section~\ref{sec:control}. Our proposed SCT is presented in Section~\ref{sec:SCT}. We comprehensively verify the efficacy of the proposed SCT for improving node classification using the three most representative GCN-style models in Section~\ref{sec:experiments}. Technical proofs and more experimental results are provided in the appendix.

\section{Preliminaries and Existing Results}\label{subsec:orthogonal-decomposition}
From the spectral graph theory \cite{chung1997spectral}, we can sort eigenvalues of matrix $\mG$ in \eqref{eq:G} in descending order as $1=\lambda_1=\ldots=\lambda_m > \lambda_{m+1}\geq\ldots\geq \lambda_n>-1$, where $m$ is the number of connected components of the graph $G$. In other words, we can group the vertex set $V =\{v_k\}_{k=1}^n$ into $m$ connected components $V_1,\ldots,V_m$. Let $\vu_i = (\vone_{\{v_k\in V_i\}})_{1\leq k\leq n}$ be the indicator vector of the $i^{th}$ component $V_i$, i.e., the $k^{th}$ coordinate of $\vu_i$ is one if the $k^{th}$ node $v_k$ lies in the connected component $V_i$; otherwise, is zero. Moreover, let $\ve_i$ be the eigenvector associated with $\lambda_i$, then $\{\ve_i\}^n_{i=1}$ forms an orthonormal basis of $\sR^n$. Notice that $\{\ve_i\}_{i=1}^{m}$ spans the eigenspace $\gM$ -- corresponding to eigenvalue 1 of matrix $\mG$, and $\{\ve_i\}_{i=m+1}^{n}$ spans the orthogonal complement of $\gM$, denoted by $\gM^\perp$. Oono and Suzuki \cite{oono2019graph} connect the indicator vectors $\vu_i$s with the space $\gM$. In particular, we have the following result:

\begin{proposition}[\cite{oono2019graph}]\label{prop:eigenvector}
All eigenvalues of the matrix $\mG$ lie in the interval $(-1, 1]$. Furthermore, the nonnegative vectors 
$
\{\Tilde{\mD}^{\frac{1}{2}}\vu_i/\|\Tilde{\mD}^{\frac{1}{2}}\vu_i\|\}_{1\leq i\leq m}
$
form an orthonormal basis of $\gM$.
\end{proposition}


For any matrix $\mH:=[\vh_1,\vh_2,\ldots,\vh_n]\in\sR^{d\times n}$, we have the decomposition $\mH = \mH_\gM + \mH_{\gM^\perp}$ with 
$\mH_\gM = \sum_{i=1}^{m}\mH\ve_i\ve_i^\top$ and $\mH_{\gM^\perp} = \sum_{i=m+1}^{n}\mH\ve_i\ve_i^\top$
such that  
$$
\langle \mH_\gM, \mH_{\gM^\perp}\rangle_F = {\rm Trace}\Bigg(\sum_{i=1}^{m}\mH\ve_i\ve_i^\top\Big(\sum_{j=m+1}^{n}\mH\ve_j\ve_j^\top\Big)^\top\Bigg) = 0,
$$
which implies that $\|\mH\|^2_F  = \|\mH_\gM\|^2_F+\|\mH_{\gM^\perp}\|^2_F.$

\subsection{Existing Smoothness Notions of Node Features}
\paragraph{Distance to the eigenspace $\gM$}
Oono and Suzuki \cite{oono2019graph} study 
the smoothness of features $\mH:=[\vh_1,\vh_2,\ldots,\vh_n]\in\sR^{d\times n}$ using their distance to the eigenspace $\gM$ as an unnormalized smoothness notion; see the definition below:
\begin{definition}[\cite{oono2019graph}]\label{def:distance-to-eigenspace}
Let $\sR^d\otimes\gM $ be the subspace of $\sR^{d\times n}$ consisting of the sum $\sum_{i=1}^{m}\vw_i\otimes\ve_i$, where $\vw_i \in \sR^d$ and $\{\ve_i\}_{i=1}^{m}$ is an orthonormal basis of the eigenspace $\gM$. Then we define $\|\mH\|_{\gM^\perp}$ -- the distance of node features $\mH$ to the eigenspace $\gM$ -- as follows: 
$$
\begin{aligned}
\|\mH\|_{\gM^\perp} &\coloneqq \inf_{\mY\in\sR^d\otimes\gM} \|\mH - \mY\|_{F}
= \big\|\mH - \sum_{i=1}^{m} \mH\ve_i\ve_i^\top\big\|_F.
\end{aligned}
$$
\end{definition}
With the decomposition $\mH=\mH_\gM+\mH_{\gM^\perp}$, the two quantities $\|\cdot\|_{\gM^\perp}$ and $\|\cdot\|_F$ can be linked via the following equation:
\begin{equation}\label{eq:distance-in-F-norm}
\begin{aligned}
\|\mH\|_{\gM^\perp} &= \inf_{\mY\in\sR^d\otimes\gM} \|\mH - \mY\|_F\\
&
= \|\mH - \mH_\gM\|_{F} \\
&= \|\mH_{\gM^\perp}\|_{F}.
\end{aligned}
\end{equation}

\paragraph{\bf Dirichlet energy} 
Cai and Wang 
\cite{cai2020note} study the unnormalized smoothness of node features using Dirichlet energy, which is defined as follows: 
\begin{definition}[\cite{cai2020note}]
Let $\Tilde{\Delta}=\mI-\mG$ be the (augmented) normalized Laplacian, then the Dirichlet energy $\|\mH\|_E$ of node features $\mH$ is defined by
$ 
\begin{aligned}
\|\mH\|^2_E&\coloneqq{\rm Trace}(\mH\Tilde{\Delta}\mH^\top).
\end{aligned}
$
\end{definition}

\paragraph{Normalized Dirichlet energy} 
Cai and Wang \cite{cai2020note} also point out that the real smoothness of node features $\mH$ should be measured by the normalized Dirichlet energy 
${\rm Trace}(\mH\Tilde{\Delta}\mH^\top)/\|\mH\|^2_F$.
This normalized smoothness measurement is essential because data often originates from various sources with diverse measurement units or scales. By normalizing the measurement, we can effectively mitigate biases resulting from these different scales.

\subsection{Two Existing Theories of Over-smoothing}\label{subsec:two-theries}
Let $\{\lambda_i\}_{i=1}^n$ be the spectrum of $\mG$, $\lambda=\max\{|\lambda_i|\mid\lambda_i<1\}$ be the second largest magnitude of $\mG$'s eigenvalues, and $s_l$ be the largest singular value of weight matrix $\mW^l$ (the weight matrix of the $l^{th}$  GCL). 
Oono and Suzuki \cite{oono2019graph} show that $\|\mH^{l}\|_{\gM^\perp}\leq s_l \lambda\|\mH^{l-1}\|_{\gM^\perp}$ under 
GCL when $\sigma$ is ReLU. Therefore, $\|\mH^l\|_{\gM^\perp}\to 0$ as $l\to\infty$ 
if $s_l\lambda<1$, 
indicating that the node features converge to the eigenspace $\gM$ and result in over-smoothing. A crucial step in the analysis in \cite{oono2019graph} is that 
$
\begin{aligned}
\|\sigma(\mZ)\|_{\gM^\perp}\leq \|\mZ\|_{\gM^\perp},
\end{aligned}
$ 
for any matrix $\mZ$ when $\sigma$ is ReLU, i.e., ReLU reduces the distance to the eigenspace $\gM$. 
Nevertheless, Oono and Suzuki \cite{oono2019graph} point out that it is hard to extend the above result to other activation functions even the leaky ReLU activation function.

Instead of considering $\|\mH\|_{\gM^\perp}$, Cai and Wang \cite{cai2020note} prove that $\|\mH^{l}\|_E\leq s_l\lambda\|\mH^{l-1}\|_E$ under GCL 
when $\sigma$ is ReLU or leaky ReLU activation function. Hence, $\|\mH^l\|_E\rightarrow 0$ as $l\rightarrow \infty$, implying over-smoothing of GCNs. Note that $\|\mH\|_{\gM^\perp}=0$ or $\|\mH^l\|_E=0$ indicates the feature vectors are homogeneous across all graph nodes. 
The proof in \cite{cai2020note} applies to GCN with both ReLU and leaky ReLU activation functions by establishing the inequality that $\|\sigma(\mZ)\|_E \leq \|\mZ\|_E $ for any matrix $\mZ$.

\section{Effects of Activation Functions: A Geometric Characterization}
\label{sec:geometry}
In this section, we present a geometric relationship between the input and output vectors of ReLU or leaky ReLU activation functions. We use $\|\mH\|_{\gM^\perp}$ as the unnormalized smoothness notion for all subsequent analyses since we observe that  $\|\mH\|_{\gM^\perp}$ and $\|\mH\|_E$ are equivalent as seminorms. In particular, we have the following results:
\begin{proposition}\label{prop:equivalent-norms}
$\|\mH\|_{\gM^\perp}$ and $\|\mH\|_E$ are two equivalent seminorms, i.e., there exist two constants $\alpha,\beta>0$ s.t. 
$
\begin{aligned}
\alpha  \|\mH\|_{\gM^\perp} \leq  \|\mH\|_E \leq  \beta \|\mH\|_{\gM^\perp},
\end{aligned}
$ 
for any $\mH \in \sR^{d\times n}$.
\end{proposition}

\subsection{ReLU}
Let $\sigma(x)=\max\{x,0\}$ be ReLU. The first main result of this paper is that there is a high-dimensional sphere underlying the input and output of ReLU. More precisely, we have the following proposition:
\begin{proposition}[ReLU]\label{prop:relu-smoothness-geometric}
For any $\mZ=\mZ_\gM+\mZ_{\gM^\perp}\in \sR^{d\times n}$, let $\mH=\sigma(\mZ)=\mH_\gM+\mH_{\gM^\perp}$.
Then
$\mH_{\gM^\perp}$ lies on the high-dimensional sphere centered at ${\mZ_{\gM^\perp}}/{2}$ with radius 
$$
{r\coloneqq \big({\|{\mZ_{\gM^\perp}}/{2}\|^2_F-
\langle \mH_{\gM}, \mH_{\gM}- \mZ_{\gM} \rangle_F}}\big)^{1/2}.
$$
In particular, 
$\mH_{\gM^\perp}$ lies inside the ball centered at ${\mZ_{\gM^\perp}}/{2}$ with radius $\|{\mZ_{\gM^\perp}}/{2}\|_F$ and hence we have $\|\mH\|_{\gM^\perp}\leq \|\mZ\|_{\gM^\perp}$.
\end{proposition}

\subsection{Leaky ReLU}
Now we consider the leaky ReLU activation function, given by $\sigma_a(x)=\max\{x,ax\}$, where $0<a<1$ is a positive scalar. Similar to ReLU, we have the following result for leaky ReLU:

\begin{proposition}[Leaky ReLU]\label{prop:leaky-relu-smoothness-geometric}
For any 
$\mZ=\mZ_\gM+\mZ_{\gM^\perp}\in\sR^{d\times n}$, 
let $\mH=\sigma_a(\mZ)=\mH_\gM+\mH_{\gM^\perp}$.
Then $\mH_{\gM^\perp}$ lies on the high-dimensional sphere centered at $  (1+a)\mZ_{\gM^\perp}/2$ with radius 
$$
\begin{aligned}
r_a\coloneqq \big(\|{{(1-a)}\mZ_{\gM^\perp}/{2}}\|^2_F-
\langle \mH_{\gM}- \mZ_{\gM} , \mH_{\gM}- a\mZ_{\gM} \rangle_F 
\big)^{1/2}. 
\end{aligned}
$$ 
In particular, 
$\mH_{\gM^\perp}$  lies inside the ball centered at $  (1+a)\mZ_{\gM^\perp}/2$ with radius $  \|{(1-a)\mZ_{\gM^\perp}}/{2}\|_F$ and hence we see that $a\|\mZ\|_{\gM^\perp}\leq \|\mH\|_{\gM^\perp} \leq  \|\mZ\|_{\gM^\perp}.$
\end{proposition}

\subsection{Implications of the Above Geometric Characterizations}
\label{subsec:implication}
Propositions \ref{prop:relu-smoothness-geometric} and \ref{prop:leaky-relu-smoothness-geometric} imply that the precise location of $\mH_{\gM^\perp}$ (or the unnormalized smoothness $\|\mH_{\gM^\perp}\|_F=\|\mH\|_{\gM^\perp}$) 
depends on the center and the radius 
$r$ or $r_a$.
Given a fixed $\mZ_{\gM^\perp}$, the center of the spheres remains unchanged, and their radii $r$ and $r_a$ are only affected by changes in $\mZ_\gM$. This observation motivates us to investigate {\bf\emph{how changes in $\mZ_\gM$ impact $\|\mH\|_{\gM^\perp}$, i.e., the unnormalized smoothness of node features}}. 

Propositions \ref{prop:relu-smoothness-geometric} and \ref{prop:leaky-relu-smoothness-geometric} imply both ReLU and leaky ReLU activation functions reduce the distance of node features to the eigenspace $\gM$, i.e. $\|\mH\|_{\gM^\perp} \leq \|\mZ\|_{\gM^\perp}$. 
Moreover, 
this inequality is independent of $\mZ_{\gM}$; consider two node features $\mZ,\mZ'\in\sR^{d\times n}$ s.t. $\mZ_{\gM^\perp}=\mZ'_{\gM^\perp}$ but $\mZ_\gM \neq \mZ'_\gM$. Let $\mH$ and $\mH'$ be the output of $\mZ$ and $\mZ'$ via ReLU or leaky ReLU activation function, respectively. Then we have the inequalities $\|\mH\|_{\gM^\perp}\leq \|\mZ\|_{\gM^\perp}$ and $\|\mH'\|_{\gM^\perp}\leq \|\mZ'\|_{\gM^\perp}$. 
Since $\mZ_{\gM^\perp}=\mZ'_{\gM^\perp}$,
we deduce that $\|\mH'\|_{\gM^\perp}\leq \|\mZ\|_{\gM^\perp}$.
In other words, when $\mZ_{\gM^\perp}=\mZ'_{\gM^\perp}$ is fixed, changing $\mZ_\gM$ to $\mZ'_\gM$ can change the unnormalized smoothness of the output features but cannot change the fact that both ReLU and leaky ReLU activation functions smooth node features;
we demonstrate this result in Fig.~\ref{fig:sec5:normalized-smoothness}a) in Section~\ref{sec:empirical-study}. Notice that without considering the nonlinear activation function, changing $\mZ_\gM$ does not affect the unnormalized smoothness of node features measured by $\|\mH\|_{\gM^\perp}$.

In contrast to the unnormalized smoothness, {\bf\emph{if one considers the normalized smoothness,
we find that adjusting $\mZ_{\gM}$ 
can result in a less smooth output}};
we will discuss this in detail in Section~\ref{sec:empirical-study}.





\section{
How Adjusting \texorpdfstring{$\mZ_\gM$}{TEXT} Affects the Smoothness of the Output}
\label{sec:control}
Throughout this section, we let $\mZ$ and $\mH$ be the input and output of ReLU or leaky ReLU activation function. The smoothness notions based on the distance of feature vectors to the eigenspace $\gM$ or their Dirichlet energy do not account for the magnitude of each dimension of the learned node features. Indeed, Cai and Wang \cite{cai2020note} points out that analyzing the normalized smoothness of features vectors $\mZ$, given by $\|\mZ\|_E/\|\mZ\|_F$, is an open problem. However, these two smoothness notions aggregate the smoothness of node features across all dimensions; when the magnitude of some dimensions is much larger than others, the smoothness will be dominated by these dimensions.

Motivated by the discussion in Section~\ref{subsec:implication}, 
we study the disparate effects of adjusting $\mZ_\gM$ on the normalized and unnormalized smoothness in this section. 
For the sake of simplicity, we assume the graph is connected ($m=1$); all the following results can be extended to graphs with multiple connected components ($m\geq 2$) easily. Due to the equivalence between seminorms $\|\cdot\|_\gM$ and $\|\cdot\|_E$, we introduce the following definition of the dimension-wise normalized smoothness of node features.

\begin{definition}\label{def:normalized-smoothness}
Let $\mZ\in\sR^{d\times n}$ be the features over $n$ nodes with $\vz^{(i)}\in\sR^n$ being its $i^{th}$ row, i.e., the $i^{th}$ dimension of the features over all nodes. We define the normalized smoothness of $\vz^{(i)}$ as follows:
$$
\begin{aligned}
s(\vz^{(i)})\coloneqq {\|\vz^{(i)}_{\gM}\|}/{\|\vz^{(i)}\|},
\end{aligned}
$$
where we set $s(\vz^{(i)}) = 1$ when $\vz^{(i)} = {\bf 0}$.
\end{definition}

\begin{remark}
Notice that the normalized smoothness $s(\vz^{(i)})={\|\vz^{(i)}_{\gM}\|}/{\|\vz^{(i)}\|}$ is closely related to the ratio between the smooth and non-smooth components of node features ${\|\vz^{(i)}_{\gM}\|}/{\|\vz^{(i)}_{\gM^\perp}\|}$. 
\end{remark}

The graph is connected ($m=1$) implies that $\vz^{(i)}_{\gM} = \langle \vz^{(i)}, \ve_1\rangle \ve_1$ and hence $\|\vz^{(i)}_{\gM}\| = |\langle \vz^{(i)}, \ve_1\rangle|$. Without ambiguity, we may drop the index and write $\vz$ for $\vz^{(i)}$ and $\ve$ for $\ve_1$ -- the only eigenvector of $\mG$ associated with the eigenvalue $1$. Moreover, we have
\begin{equation} \label{eq:node-feature-smoothness}
\begin{aligned}
s(\vz) = \frac{\|\vz_{\gM}\|}{\|\vz\|} = \frac{|\langle \vz, \ve \rangle|}{\|\vz\|} = \frac{|\langle \vz, \ve \rangle|}{\|\vz\|\cdot\|\ve\|} \Rightarrow 0 \leq s(\vz) \leq 1,
\end{aligned}
\end{equation}
It is evident that the larger $s(\vz)$ is, the smoother the node feature $\vz$ is\footnote{Here, $\vz\in\sR^n$ is a vector whose $i^{th}$ entry is the 1D feature associated with node $i$.}. In fact, we have
$$ 
\begin{aligned}
  s(\vz)^2 + \Big(\frac{\|\vz\|_{\gM^\perp}}{\|\vz\|}\Big)^2 =  \frac{\|\vz_{\gM}\|^2}{\|\vz\|^2} + \frac{\|\vz_{\gM^\perp}\|^2}{\|\vz\|^2} = 1,
\end{aligned}
$$
where 
$\|\vz\|_{\gM^\perp}/\|\vz\|$ decreases as $s(\vz)$ increases.

\begin{figure}[!ht]
\begin{center}
\begin{tabular}{cc}
\includegraphics[width=0.48\linewidth]{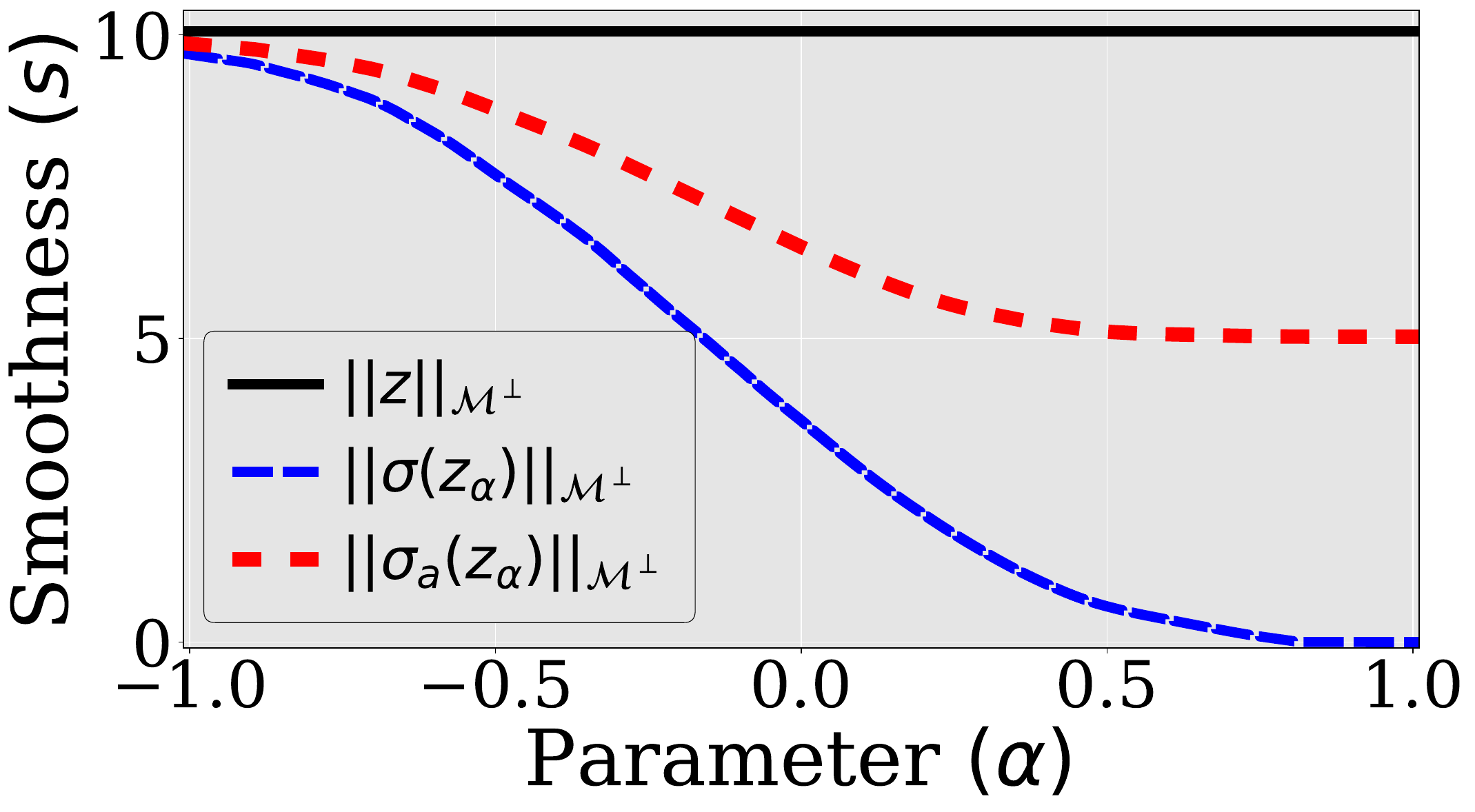}&
\includegraphics[width=0.48\linewidth]{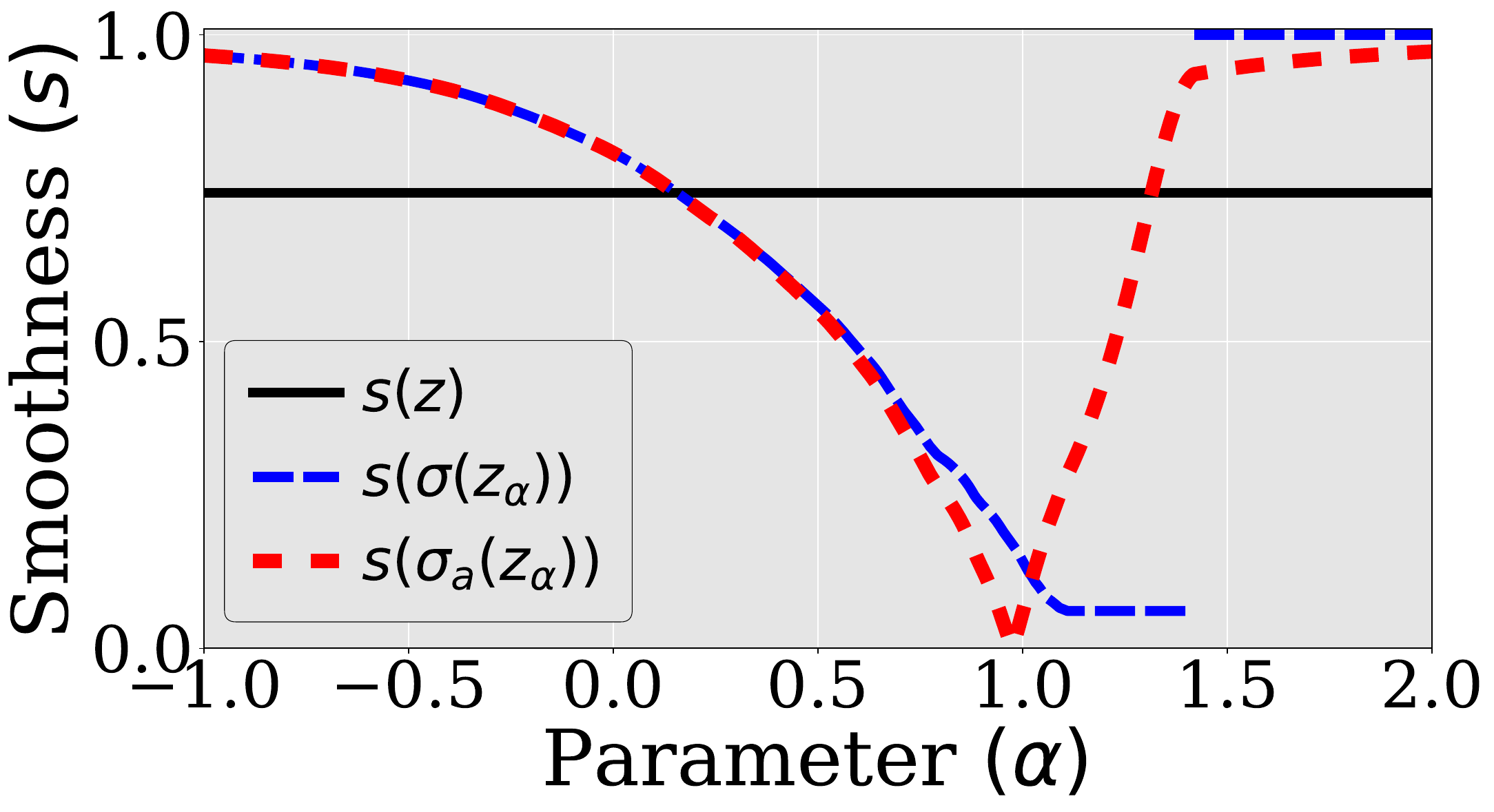}\\
{ a) Smoothness} & { b) Normalized smoothness}
\end{tabular}
  \end{center}
  \caption{
  Contrasting the effects of varying parameter $\alpha$ on the smoothness and normalized smoothness of output node 
  features $\sigma(\vz_\alpha)$ and $\sigma_a(\vz_\alpha)$. Notice that the discontinuity of $s(\sigma(\vz_\alpha))$ in panel 
  b) comes from the definition of normalized smoothness. Moreover, note that $s(\vz)=1$ if $\vz=\bf 0$, and $\sigma(\vz_\alpha)$ can become $\bf 0$ when $\alpha$ is large enough. 
  }\label{fig:sec5:normalized-smoothness}
\end{figure}
To discuss how the smoothness $s(\vh) = s(\sigma(\vz))$ or $s(\sigma_a(\vz))$ can be adjusted by changing $\vz_\gM$, we consider the following function: 
$$
\vz(\alpha) = \vz - \alpha\ve.
$$ 
It is clear that 
$$
\begin{aligned}
\vz(\alpha)_{\gM^\perp} = \vz_{\gM^\perp} \text{ and } \vz(\alpha)_{\gM} = \vz_{\gM} - \alpha\ve,
\end{aligned}
$$
where we see that $\alpha$ only 
alters $\vz_\gM$ 
while preserves $\vz_{\gM^\perp}$. Moreover, it is evident that
\begin{equation*} 
\begin{aligned}s(\vz(\alpha)) = \sqrt{1 - \frac{\|\vz(\alpha)_{\gM^\perp}\|^2}{\|\vz(\alpha)\|^2}} = \sqrt{1 - \frac{\|\vz_{\gM^\perp}\|^2}{\|\vz(\alpha)\|^2}}.
\end{aligned}
\end{equation*}
It follows that ${  s(\vz(\alpha))=1}$ if and only if ${  \vz_{\gM^\perp}=\bf 0}$ (include the case ${  \vz=\bf 0}$), 
{showing that when $\vz_{\gM^\perp}=\bf 0$, the vector $\vz$ is the smoothest one.}

\subsection{The Disparate Effects of \texorpdfstring{$\alpha$}{TEXT} on 
\texorpdfstring{$\|\cdot\|_{\gM^\perp}$}{TEXT} and \texorpdfstring{$s(\cdot)$}{TEXT}: Empirical results}\label{sec:empirical-study}
Let us conduct a simple empirical study to investigate possible values that the unnormalized smoothness $\|\sigma(\vz(\alpha))\|_{\gM^{\perp}}$, $\|\sigma_a(\vz(\alpha))\|_{\gM^{\perp}}$ and 
the normalized smoothness $s(\sigma(\vz(\alpha)))$, $s(\sigma_a(\vz(\alpha)))$ can take when $\alpha$ varies. 
In this subsection, we denote $\vz_\alpha\coloneqq\vz(\alpha) = \vz-\alpha\ve$. We consider a connected synthetic graph with $100$ nodes, and each node is assigned a random degree between $2$ to $10$. Then we assign an initial node feature vector $\vz\in\sR^{100}$, sampled uniformly on the interval $[-1.5,1.5]$, to the graph with each node feature being a scalar. Also, we compute $\ve$ by the formula $\ve = \Tilde{\mD}^{\frac{1}{2}}\vu/\|\Tilde{\mD}^{\frac{1}{2}}\vu\|$ from Proposition~\ref{prop:eigenvector}, where $\vu\in\sR^{100}$ is the vector whose entries are all ones and $\Tilde{\mD}$ is the (augmented) degree matrix of the graph. 
We examine two different smoothness notions for the input node features $\vz$ and the output node features $\sigma(\vz_\alpha)$ and $\sigma_a(\vz_\alpha)$, where the smoothness is measured for various values of the smoothness control parameter $\alpha$ in the range $[-1.5, 1.5]$. 
In Fig.~\ref{fig:sec5:normalized-smoothness}a), we study the unnormalized smoothness measured by $\|\cdot\|_{\gM^{\perp}}$; we see that $\|\sigma(\vz_\alpha)\|_{\gM^{\perp}}$ and $\|\sigma_a(\vz_\alpha)\|_{\gM^{\perp}}$, the smoothness of the output nodes features, are always no greater than $\|\vz\|_{\gM^{\perp}}$ -- the smoothness of the input. 
This coincides with the discussion in Section~\ref{subsec:implication}; adjusting the projection of $\vz$ onto the eigenspace $\gM$ can not change the fact that $\|\sigma(\vz_\alpha)\|_{\gM^{\perp}} \leq \|\vz\|_{\gM^{\perp}}$ and $\|\sigma_a(\vz_\alpha)\|_{\gM^{\perp}} \leq \|\vz\|_{\gM^{\perp}}$. Nevertheless, an interesting result is that {\bf\emph{altering the eigenspace projection can adjust the unnormalized smoothness of the output}}: notice that altering the eigenspace projection does not change its distance to $\gM$, i.e., the smoothness of the input is unchanged, but the smoothness of the output after activation function can be changed.

In contrast, when studying the normalized smoothness measured by $s(\cdot)$ in Fig.~\ref{fig:sec5:normalized-smoothness}b), we find that $s(\sigma(\vz(\alpha)))$ and $s(\sigma_a(\vz(\alpha)))$, the normalized smoothness of the output nodes features, can be adjusted by $\alpha$ to values smaller than $s(\vz)$. 
More precisely, we see that by adjusting $\alpha$, $s(\sigma(\vz(\alpha)))$ and $s(\sigma_a(\vz(\alpha)))$ can achieve most of the values in $[0,1]$. In other words, both smoother and less smooth features can be obtained by adjusting $\alpha$.

\subsection{Theoretical Results on the Smooth Effects of ReLU and Leaky ReLU}
In this subsection, we build theoretical understandings of the above empirical findings on the achievable smoothness of $\sigma(\vz(\alpha))$ and $\sigma_a(\vz(\alpha))$ by adjusting $\alpha$ -- shown in Fig.~\ref{fig:sec5:normalized-smoothness}.
Notice that if $\vz_{\gM^\perp}={\bf 0}$, the inequalities presented in Propositions~\ref{prop:relu-smoothness-geometric} and \ref{prop:leaky-relu-smoothness-geometric} indicate that $\|\sigma(\vz(\alpha))\|_{\gM^\perp}$ and $\|\sigma_a(\vz(\alpha))\|_{\gM^\perp}$ 
vanish. So we have 
$s(\sigma(\vz(\alpha))) = 1$ for any $\alpha$ when $\vz_{\gM^\perp} = {\bf 0}$. Then we may assume $\vz_{\gM^\perp} \neq {\bf 0}$ for the following study.

\begin{proposition}[ReLU] 
\label{prop:smoothness-control-relu} 
Suppose $\vz_{\gM^\perp} \neq {\bf 0}$. Let $\vh(\alpha)=\sigma(\vz(\alpha))$ with $\sigma$ being ReLU, then we have
\[ 
\begin{aligned}
\min_\alpha s(\vh(\alpha)) = \sqrt{\frac{\sum_{x_i = \max \vx}d_i}{\sum^n_{j =1}d_j}},
\end{aligned}
\]
and 
\[ 
\begin{aligned}
\max_\alpha s(\vh(\alpha)) = 1,
\end{aligned}
\]
where $\vx\coloneqq\Tilde{\mD}^{-\frac{1}{2}}\vz$, $\max\vx=\max_{1\leq i\leq n}x_i$, and $\Tilde{\mD}$ is the augmented degree matrix, whose diagonal entries are $d_1,d_2,\ldots,d_n$.
In particular, the normalized smoothness $s(\vh(\alpha))$ is monotone increasing as $\alpha$ decreases whenever $\alpha< \|\Tilde{\mD}^{\frac{1}{2}}\vu_n\| \max\vx$ and it has range $[\min_\alpha s(\vh(\alpha)), 1]$.
\end{proposition}

\begin{proposition}[Leaky ReLU] 
\label{prop:smoothness-control-leaky-relu}
Suppose $\vz_{\gM^\perp} \neq \bf 0$. Let $\vh(\alpha)=\sigma_a(\vz(\alpha))$ with $\sigma_a$ being leaky ReLU, then (1) $\min_\alpha s(\vh(\alpha)) = 0$, and (2) $\sup_\alpha s(\vh(\alpha)) = 1$ 
and $s(\vh(\alpha))$ has range $[0, 1)$.
\end{proposition}

Proposition~\ref{prop:smoothness-control-leaky-relu} 
also 
holds for other variants of ReLU, e.g., ELU\footnote{The ELU function is defined by $f(x) = \max(x, 0) + \min(0,a\cdot (e^x -1))$ where $a>0$.} and SELU\footnote{The SELU function is defined by $f(x) = c(\max(x, 0) + \min(0,a\cdot (e^x -1 )))$ where $a,c>0$.}.; see 
\ref{sec:sct}. 
%
%
We summarize Propositions~\ref{prop:relu-smoothness-geometric}, \ref{prop:leaky-relu-smoothness-geometric}, \ref{prop:smoothness-control-relu}, and \ref{prop:smoothness-control-leaky-relu} in the following corollary, which qualitatively explains the empirical results in Fig.~\ref{fig:sec5:normalized-smoothness}.

\begin{corollary}\label{cor:smoothness effective control}
Suppose $\vz_{\gM^\perp} \neq {\bf 0}$. Let $\vh(\alpha)=\sigma(\vz(\alpha))$ or $\sigma_a(\vz(\alpha))$ with $\sigma$ being ReLU and $\sigma_a$ being leaky ReLU. Then we have $\|\vz\|_{\gM^\perp} 
\geq \|\vh(\alpha)\|_{\gM^\perp}$ for any $\alpha\in\sR$; however, $s(\vh(\alpha))$ is not always smaller than $s(\vz)$. In particular, 
$s(\vh(\alpha))$ can be smaller than, larger than, or equal to $s(\vz)$ for different values of $\alpha$.
\end{corollary}

Propositions \ref{prop:smoothness-control-relu} and \ref{prop:smoothness-control-leaky-relu}, as well as Corollary \ref{cor:smoothness effective control}, provide a theoretical basis for the empirical results presented in Fig.~\ref{fig:sec5:normalized-smoothness}. Moreover, the above theoretical results indicate that for any given vector $\vz$, altering its projection $\vz_\gM$ can effectively change both the unnormalized and the normalized smoothness of the output vector $\vh=\sigma(\vz)$ or $\sigma_a(\vz)$. In particular, the normalized smoothness of the output vector 
$\vh=\sigma(\vz)$ or $\sigma_a(\vz)$ can be adjusted to any value in the range shown in Propositions~\ref{prop:smoothness-control-relu} and \ref{prop:smoothness-control-leaky-relu}. This provides us with insights to design algorithms to control the smoothness of feature vectors to improve the performance of GCN 
control the normalized smoothness of each dimension of the feature vectors, and we will discuss this in the next section.

\section{Controlling Smoothness of Node Features}
\label{sec:SCT}
We do not know how smooth features are ideal for a given node classification task. Nevertheless, our theory indicates that both normalized and unnormalized smoothness of the output of each GCL can be adjusted by altering the input's projection onto $\gM$. 
As such, we propose the following learnable smoothness control term to modulate the smoothness of each dimension of the learned node features: 
\begin{equation} 
\label{eq:smoothness-control-term}
\begin{aligned}
\mB^{l}_{\bm\alpha} = \sum_{i=1}^m {\bm\alpha}_i^l \ve_i^\top,
\end{aligned}
\end{equation}
where $l$ is the layer index, $\{\ve_i\}_{i=1}^m$ is the orthonormal basis of the eigenspace $\gM$ -- provided in Proposition~\ref{prop:eigenvector}, 
and $\bm\alpha^l:=\{{\bm\alpha}_i^l\}_{i=1}^m$ is a collection of learnable vectors with ${\bm\alpha}_i^l\in\sR^d$ being approximated by a multi-layer perceptron (MLP).
The detailed configuration of ${\bm\alpha}_i^l$ will be specified in each experiment later. One can see that $\mB^{l}_{\bm\alpha}$ always lies in $\sR^{d}\otimes\gM$.
We integrate SCT into GCL, resulting in the following update equation:
\begin{equation}
\label{eq:bias-design}
    \begin{aligned}
        \mH^{l} &= \sigma(\mW^{l}\mH^{l-1}\mG+\mB^{l}_{\bm\alpha}).
    \end{aligned}
\end{equation}
We call the corresponding model GCN-SCT. 
Again, the idea here is that {\bf\emph{we alter the component in eigenspace to control the smoothness of node features}}. In particular, each dimension of the output $\mH^l$ of \eqref{eq:bias-design} can be smoother, less smooth, or the same as that of $\mH^{l-1}$ in terms of normalized smoothness, though $\mH^l$ 
gets closer to the eigenspace $\gM$ than $\mH^{l-1}$.

Next, we elaborate on the proposed SCT. To design SCT, we introduce a learnable matrix $\mA^l\in\sR^{d\times m}$ for layer $l$, whose columns are ${\bm\alpha}^l_i$, where $m$ is the dimension of the eigenspace $\gM$ and $d$ is the dimension of the features. We observe in our experiments that the SCT performs best when informed by degree pooling over the subcomponents of the graph. The matrix of the orthogonal basis vectors, denoted by $\mQ\coloneqq[\ve_1,\ldots,\ve_m]\in\sR^{n\times m}$, is used to perform pooling $\mH^l\mQ$ for input $\mH^l$. 
In particular, for the first architecture, we let ${\mA^l}=\mW\odot(\mH^l\mQ)$, where  $\mW\in\sR^{d\times m}$ is learnable and performs pooling over $\mH^l$ using the eigenvectors $\mQ$. 
The second architecture uses a residual connection with hyperparameter $\beta_l=\log({\theta}/{l}+1)$ and 
learnable matrices $\mW_0,\mW_1\in\sR^{d\times d}$ and the softmax function $\phi$. Resulting in 
$
\begin{aligned}
{\mA^l}=\phi(\mH^l\mQ)\odot(\beta_l\mW_0\mH^0\mQ + (1-\beta_l)\mW_1\mH^l\mQ)
\end{aligned}
$.
In Section~\ref{sec:experiments}, we use the first architecture for GCN-SCT as GCN uses only $\mH^l$ information at each layer. We use the second architecture for GCNII-SCT and EGNN-SCT which use both $\mH^0$ and $\mH^l$ information at each layer. There are two particular advantages of the above design of SCT: (1) it can effectively change the normalized smoothness of the learned node features, and (2) it is computationally efficient since we only use the eigenvectors corresponding to the eigenvalue 1 of the message-passing matrix $\mG$, which is determined based on the connectivity of the graph.



\begin{remark}
It is worth noting that computing the basis of eigenspace $\gM$ -- corresponding to the eigenvalue 1 of the message-passing matrix $\mG$ -- does not introduce substantial computational overhead. In particular, the basis of the space $\mathcal{M}$ is given by the indicator functions of each connected component of the graph. Therefore, the problem reduces to finding connected components of the graph, and we can identify connected components for undirected graphs using disjoint set union (DSU) \cite{galil1991data}. Initially, declare all the nodes as individual subsets and then visit them. When a new unvisited node is encountered, unite it with the under. In this manner, a single component will be visited in each traversal. The time complexity is linear with respect to the number of nodes.
\end{remark}

\subsection{Integrating SCT into Other GCN-style Models}
In this subsection, we present other usages of the proposed SCT. 
We carefully select two other most representative usages of the proposed SCT. The first example is GCNII \cite{chen2020simple}, GCNII extends GCN to express an arbitrary polynomial filter rather than the Laplacian polynomial filter and has been shown to achieve state-of-the-art (SOTA) performance among GCN-style models on various benchmark tasks \cite{chen2020simple,luan2022revisiting}, and we aim to show that the proposed SCT -- that enables GCNII to learn node features with a better smoothness -- can even benefit node classification for improving the accuracy of the GCN-style model that achieves SOTA performance on many node classification tasks. 
The second example is energetic GNN (EGNN) \cite{zhou2021dirichlet}, which controls the smoothness of node features by constraining the lower and upper bounds of the Dirichlet energy of node features and assuming the activation function is linear. However, in practice, the activation function is ReLU or possibly other nonlinear functions. In this case, we aim to show that our established new theoretical understanding of the role of activation functions and the proposed SCT can boost the performance of EGNN with consideration of nonlinear activation functions.

\medskip
\noindent{\bf GCNII.}
Each GCNII layer uses a skip connection to the initial layer $\mH^0$ and 
given as follows:
\begin{equation*}
\begin{aligned}
\mH^{l} = \sigma\big(((1-\alpha_l)\mH^{l-1}\mG + \alpha_l\mH^{0}) ((1-\beta_l)\mI + \beta_l\mW^{l})\big),
\end{aligned}
\end{equation*}
where $\alpha_l,\beta_l\in(0,1)$ are learnable scalars.
%
We integrate SCT $\mB^{l}_{\bm\alpha}$ into GCNII, resulting in the following GCNII-SCT layers:
\begin{equation*}
\begin{aligned}
\mH^{l} = \sigma\big(( (1-\alpha_l)\mH^{l-1}\mG + \alpha_l\mH^{0} )
            ((1-\beta_l)\mI + \beta_l\mW^{l}) +\mB^{l}_{\bm\alpha}\big),
\end{aligned}
\end{equation*}
where the residual connection and identity mapping are consistent with GCNII. We call the resulting model GCNII with a smoothness control term (GCNII-SCT).


\medskip
\noindent{\bf EGNN.}
EGNN \cite{zhou2021dirichlet} controls the smoothness of node features by constraining the lower and upper bounds of the Dirichlet energy of node features without considering the nonlinear activation function.
Each EGNN layer can be written as follows:
\begin{equation}
\begin{aligned}
        \mH^{l} = \sigma\big(\mW^{l}(c_1\mH^{0} + c_2\mH^{l-1} + (1-c_{\min})\mH^{l-1}\mG)\big),
\end{aligned}
\end{equation}
where $c_1,c_2$ are learnable weights that satisfy $c_1+c_2 = c_{\min}$ with $ c_{\min}$ being a hyperparameter.
%
To constrain Dirichlet energy, EGNN initializes trainable weights $\mW^{l}$ as a diagonal matrix with explicit singular values and regularizes them to keep the orthogonality during the model training. Ignoring the activation function $\sigma$, $\mH^l$ -- node features at layer $l$ of EGNN satisfies:
$$
\begin{aligned}
c_{\min}\|\mH^0\|_E \leq \|\mH^l\|_E  \leq c_{\max}\|\mH^0\|_E,
\end{aligned}
$$
where $c_{\max}$ is {the square of the maximal singular value of the initialization of $\mW^1$}.
%
Similarly, we modify EGNN to result in the following EGNN-SCT layer:
\begin{equation*}
\begin{aligned}
\mH^{l} = \sigma\big(\mW^{l}((1-c_{\min})\mH^{l-1}\mG+ c_1\mH^{0} + c_2\mH^{l-1} )+\mB^{l}_{\bm\alpha}\big),
\end{aligned}
\end{equation*}
where everything remains the same as the EGNN layer except that we include our proposed SCT $\mB^{l}_{\bm\alpha}$. 


\section{Experiments}\label{sec:experiments}
In this section, we comprehensively demonstrate the effects of SCT -- in the three most representative GCN-style models discussed in Section~\ref{sec:SCT} -- using various node classification benchmarks. The purpose of all experiments in this section is to verify the efficacy of the proposed SCT 
-- motivated by our theoretical results -- 
for GCN-style models. 
Exploring the effects of SCT on non-GCN-style models and pushing for SOTA accuracy is an interesting future direction. 
We consider the citation datasets (Cora, Citeseer, PubMed, Coauthor-Physics, Ogbn-arxiv), web knowledge-base datasets (Cornell, Texas, Wisconsin), and Wikipedia network datasets (Chameleon, Squirrel). We provide additional dataset details in 
\ref{appendix:datasets}. We implement baseline GCN \cite{kipf2017semisupervised} 
and GCNII \cite{chen2020simple} 
(without weight sharing) using PyG (Pytorch Geometric) 
\cite{Fey-Lenssen-2019}. Baseline EGNN \cite{zhou2021dirichlet} 
is implemented using the public 
code\footnote{https://github.com/Kaixiong-Zhou/EGNN}. 

\subsection{
Node Feature Trajectory}\label{sec:experiments:smoothness}
\begin{figure}[!ht]
\begin{center}
\begin{tabular}{ccc}
\includegraphics[width=0.31\linewidth]{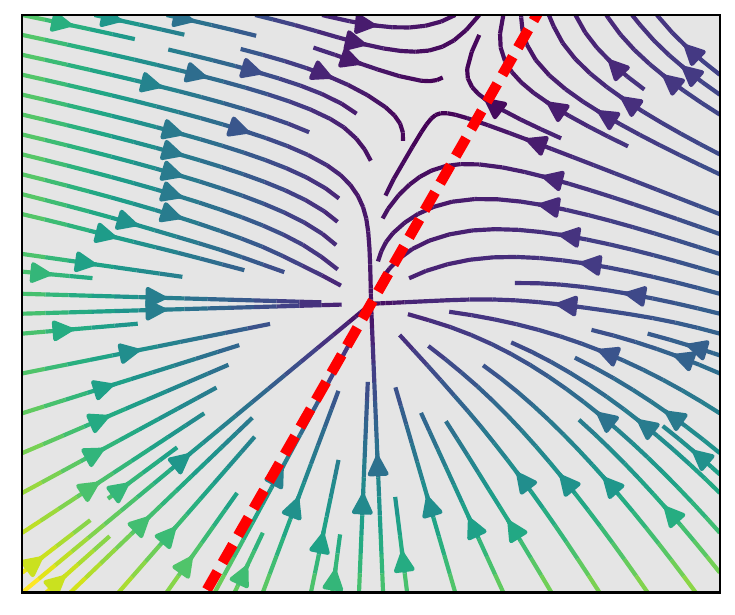}&
\includegraphics[width=0.31\linewidth]{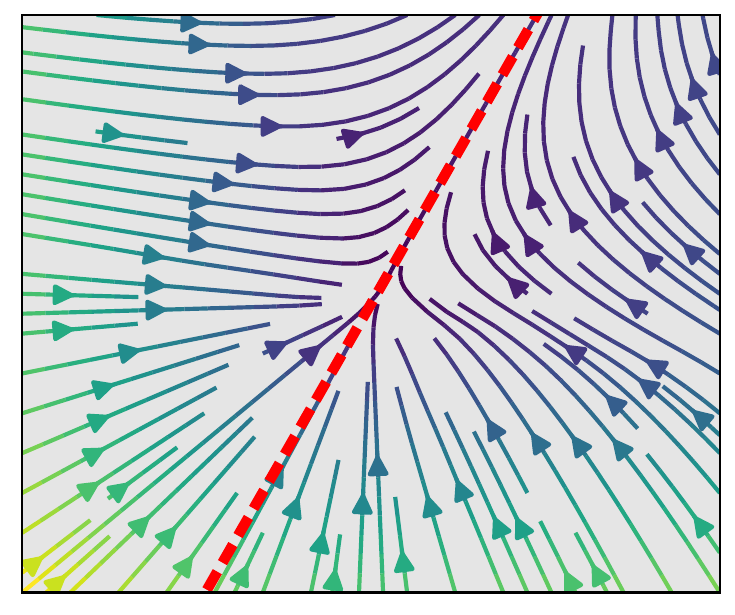}&
\includegraphics[width=0.31\linewidth]{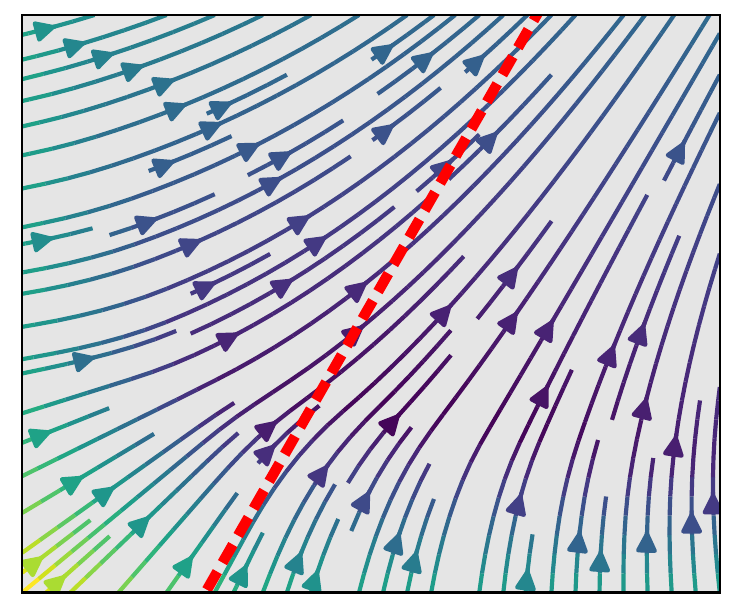}\\ 
{\small  a) $\alpha=-0.25$} & {\small  b) $\alpha=0.0$} & {\small  c) $\alpha=1.0$}\\
  \end{tabular}
  \end{center}
  \caption{ 
Node feature trajectories, with colorized magnitude, for varying smoothness control parameter $\alpha$. For classical GCN b), the node features converge to the eigenspace $\gM$ (red dashed line).
  }\label{fig:sec5:trajectory}
\end{figure}
We visualize the trajectory of the node features, following \cite{oono2019graph}, for a graph with two nodes connected by an edge and 1D node feature. In this case, \eqref{eq:bias-design} becomes $\vh^1=\sigma(w\vh^0\mG+\vb_\alpha)$, where $w=1.2$ in our experiment, 
$\vh^0,\vh^1,\vb_{\alpha}\in\sR^2$, and $\mG\in\sR^{2\times 2}$. We use a positive definite matrix $\mG=[0.592, 0.194; 0.194, 0.908]$ whose largest eigenvalue is 1. 
Twenty initial node feature vectors $\vh^0$ are sampled evenly in the domain $[-1,1]\times[-1,1]$. Figure \ref{fig:sec5:trajectory} shows the trajectories in relation to the eigenspace $\gM$ (red dashed line). In Fig~\ref{fig:sec5:trajectory}a), one can see that some trajectories do not directly converge to $\gM$.
In Fig.~\ref{fig:sec5:trajectory}b) when $\alpha=0.0$, 
GCL is recovered and all trajectories converge to $\gM$. In Fig.~\ref{fig:sec5:trajectory}c), large positive values of $\alpha$ (e.g. $1.0$) enable the node features to significantly deviate from $\gM$ initially. We observe that the parameter $\alpha$ can effectively change the trajectory of features.

\begin{table}[!ht]
\fontsize{8.5}{8.5}\selectfont
\centering
\begin{tabular}{c|ccccc}
\specialrule{1.2pt}{1pt}{1pt}
     Layers & 
     2
     &
     4
     &
     16
     &
     32
     \\
 \specialrule{1.2pt}{1pt}{1pt}
     \multicolumn{5}{c}{\textbf{Cora}}\\
     \hline
         GCN
         /GCN-SCT & $81.1$/
         {$\mathbf{82.9}$} & $80.4$/
         {$\mathbf{82.8}$} 
         & $64.9$/
         {$\mathbf{71.4}$}  & $60.3$/
         {$\mathbf{67.2}$}  
         \\
         GCNII
         /GCNII-SCT & $82.2$/
         {$\mathbf{83.8}$} & $82.6$/
         {$\mathbf{84.3}$} 
         & 
         {{$84.6$}/{$\mathbf{84.8}$}} & {${85.4}$}/
         {$\mathbf{85.5}$} 
         \\
         EGNN
         /EGNN-SCT & $83.2$/$\mathbf{84.1}$ & {$84.2$/$\mathbf{84.5}$}
         & 
         {{$\mathbf{85.4}$}/$83.3$} & 
         {$\mathbf{85.3}$/$82.0$} 
         \\
\specialrule{1.2pt}{1pt}{1pt}
     \multicolumn{5}{c}{\textbf{Citeseer}}\\
 \hline
         GCN/GCN-SCT & 
          {{$\mathbf{70.3}$}/$69.9$}
         & $67.6$/
         {$\mathbf{67.7}$} 
         & $18.3$/
         {$\mathbf{55.4}$} & $25.0$/
         {$\mathbf{51.0}$} 
         \\
         GCNII/GCNII-SCT & $68.2$/
         {$\mathbf{72.8}$} & $68.9$/
         {$\mathbf{72.8}$} 
         & {$72.9$}/
         {$\mathbf{73.8}$} & 
         {{$\mathbf{73.4}$}/{{$\mathbf{73.4}$}}} 
         \\
         EGNN/EGNN-SCT & $72.0$/$\mathbf{73.1}$ & $71.9$/$\mathbf{72.0}$ 
         & $72.4$/$\mathbf{72.6}$  & $72.3$/$\mathbf{72.9}$\\ 
\specialrule{1.2pt}{1pt}{1pt}
     \multicolumn{5}{c}{\textbf{PubMed}}\\
 \hline
         GCN/GCN-SCT &  $79.0$/
         {$\mathbf{79.8}$} & $76.5$/
         {$\mathbf{78.4}$} 
         & $40.9$/
         {$\mathbf{76.1}$} & $22.4$/
         {$\mathbf{77.0}$}  
         \\
         GCNII/GCNII-SCT & $78.2$/
         {$\mathbf{79.7}$} & $78.8$/
         {$\mathbf{80.1}$} 
         & $80.2$/
         {$\mathbf{80.7}$} & $79.8$/
         {$\mathbf{80.7}$} 
         \\
         EGNN/EGNN-SCT & $79.2$/$\mathbf{79.8}$ & $79.5$/$\mathbf{80.4}$ 
         & $80.1$/$\mathbf{80.3}$ & $80.0$/$\mathbf{80.4}$ 
         \\ 
    \specialrule{1.2pt}{1pt}{1pt}
     \multicolumn{5}{c}{\textbf{Coauthor-Physics}}\\
 \hline
         GCN/GCN-SCT & $92.4$/
         {$\mathbf{92.6\pm1.6}$}  & $92.1$/
         {$\mathbf{92.5\pm5.9}$}  
         &  $13.5$/
         {$\mathbf{50.9\pm15.0}$} &   $13.1$/
         {$\mathbf{43.6\pm16.0}$} 
         \\
         GCNII/GCNII-SCT & $92.5$/
         {$\mathbf{94.4\pm0.4}$} & $92.9$/
         {$\mathbf{94.2\pm0.3}$} 
         & $92.9$/
         {$\mathbf{93.7\pm0.7}$} &  $92.9$/
         {$\mathbf{94.1\pm0.3}$} 
         \\
         EGNN/EGNN-SCT & $92.6$/$\mathbf{93.9\pm0.7}$ & $92.9$/$\mathbf{94.1\pm0.4}$ 
         & $93.1$/$\mathbf{94.0\pm0.7}$  &  $93.3$/$\mathbf{93.8\pm1.3}$  
         \\
    \specialrule{1.2pt}{1pt}{1pt}
     \multicolumn{5}{c}{\textbf{Ogbn-arxiv}}\\
 \hline
         GCN/GCN-SCT & $70.4$/
         {$\mathbf{72.1\pm0.3}$}  &  $71.7$/
         {$\mathbf{72.7\pm0.3}$} 
         & $70.6$/
         {$\mathbf{72.3\pm0.2}$} & $68.5$/
         {$\mathbf{72.3\pm0.3}$}  
         \\
         GCNII/GCNII-SCT & $70.1$/
         {$\mathbf{72.0\pm0.3}$}  & $71.4$/
         {$\mathbf{72.2\pm0.2}$} 
         & $71.5$/
         {$\mathbf{72.4\pm0.3}$} &  $70.5$/
         {$\mathbf{72.1\pm0.3}$} 
         \\
         EGNN/EGNN-SCT &  $68.4$/$\mathbf{68.5\pm0.6}$ & $71.1$/{$\mathbf{71.3\pm0.5}$} 
         & $72.7$/$\mathbf{72.8\pm0.5}$ &  $\mathbf{72.7}$/$72.3\pm0.5$ 
         \\
    \specialrule{1.2pt}{1pt}{1pt}
    \end{tabular}
\caption{
Accuracy for models of varying depth on citation datasets. We note vanishing gradients occur but not {over-smoothing} for the cases of accuracy drop using GCN-SCT with 16 or 32 layers. For Cora, Citeseer, and PubMed, we use a fixed split with a single forward pass 
following \cite{chen2020simple}; only test accuracy is available in these experiments. For Coauthor-Physics and Ogbn-arxiv, we use the splits from \cite{zhou2021dirichlet}; both test accuracy and standard deviation are reported. The baseline results are copied from \cite{chen2020simple,zhou2021dirichlet} where the standard deviation was not reported.
(Unit:\%)}
    \label{table:node:variable-layers}
\end{table}

\subsection{Baseline Comparisons for Node Classification}\label{subsec-node-classification}


\subsubsection{Citation Networks}
We compare the three representative GCN-style models discussed in Section~\ref{sec:SCT}, of different depths, with and without SCT in Table~\ref{table:node:variable-layers}. This task uses the citation datasets with fixed splits from \cite{yang2016revisiting} for Cora, Citeseer, and Pubmed and splits from \cite{zhou2021dirichlet} for Coauthor-Physics and Ogbn-arxiv; a detailed description of these datasets and splits are provided in 
\ref{appendix-exp-details}. Following \cite{chen2020simple}, we use a single training pass to minimize the negative log-likelihood loss using the Adam optimizer \cite{kingma2014adam}, with $1500$ maximum epochs, and $100$ epochs of patience. A grid search for possible hyperparameters is listed in Table~\ref{table:hyperparameter:gridsearch} in 
\ref{appendix-exp-details}. We accelerate the hyperparameter search by applying a Bayesian meta-learning algorithm~\cite{wandb} which minimizes the validation loss, and we run the search for $200$ iterations per model. In particular, Table~\ref{table:node:variable-layers} presents the best test accuracy between ReLU and leaky ReLU activation functions for GCN, GCNII, and all three models with SCT\footnote{A comparison of the results using ReLU and leaky ReLU activation functions is presented in \ref{appendix-exp-details}.}. 
For the baseline EGNN, we follow \cite{zhou2021dirichlet} using the SReLU activation function, a particular activation used for EGNN in \cite{zhou2021dirichlet}. These results show that SCT can boost the classification accuracy of baseline models; in particular, the improvement can be remarkable for GCN and GCNII. However, EGNN-SCT (using ReLU or leaky ReLU activation function) performs occasionally worse than EGNN (using SReLU), and this is because of the choice of activation functions. 
In 
\ref{appendix:semi-sup}, we report the results of EGNN-SCT using SReLU, showing that EGNN-SCT outperforms EGNN in all tasks.
In fact, the SReLU activation function is a shifted version of ReLU, and our theory for ReLU applies to SReLU as well. In \ref{appendix-addition-results}, we perform hypothesis tests to show the statistical significance of the accuracy improvement, especially for the Cora, Citeseer, and PubMed datasets (since the standard deviation of the baseline results is not available). The accuracy improvement for the other two datasets usually exceeds the standard deviation by a significant margin. The model size and computational time are reported in Table~\ref{table:model_sizes} in the appendix. 
 


Table~\ref{table:node:variable-layers} also shows that even with SCT, the 
accuracy of GCN drops when the depth is 16 or 32. This motivates us to investigate the smoothness of the node features learned by GCN and GCN-SCT. 
Fig.~\ref{fig:smoothness-GCN-GCNSCT} plots the heatmap of the normalized smoothness of each dimension of the learned node features learned by GCN and GCN-SCT with 32 layers for Citeseer node classification. In these plots, the horizontal and vertical dimensions denote the feature dimension and the layer of the model, respectively. 
We notice that the normalized smoothness of each dimension of the features -- from {layers 14 to 32} learned by GCN -- closes to 1, confirming that deep GCN learns homogeneous features. In contrast, the features learned by GCN-SCT are inhomogeneous, as shown in Fig.~\ref{fig:smoothness-GCN-GCNSCT}b). Therefore, we believe the performance degradation of deep GCN-SCT is due to other factors. 
Compared to GCNII/GCNII-SCT and EGNN/EGNN-SCT, GCN-SCT does not use skip connections, which is known to help avoid vanishing gradients in training deep neural networks \cite{He_2016_CVPR,he2016identity}. In 
\ref{appendix:semi-sup}, we show that training GCN and GCN-SCT do suffer from the vanishing gradient issue; however, the other models do not. Besides Citeseer, we notice similar behavior occurs for training GCN and GCN-SCT for Cora and Coauthor-Physics node classification tasks. 

\begin{figure}[!ht]
\begin{center}
\begin{tabular}{cc}
\includegraphics[width=0.45\linewidth]{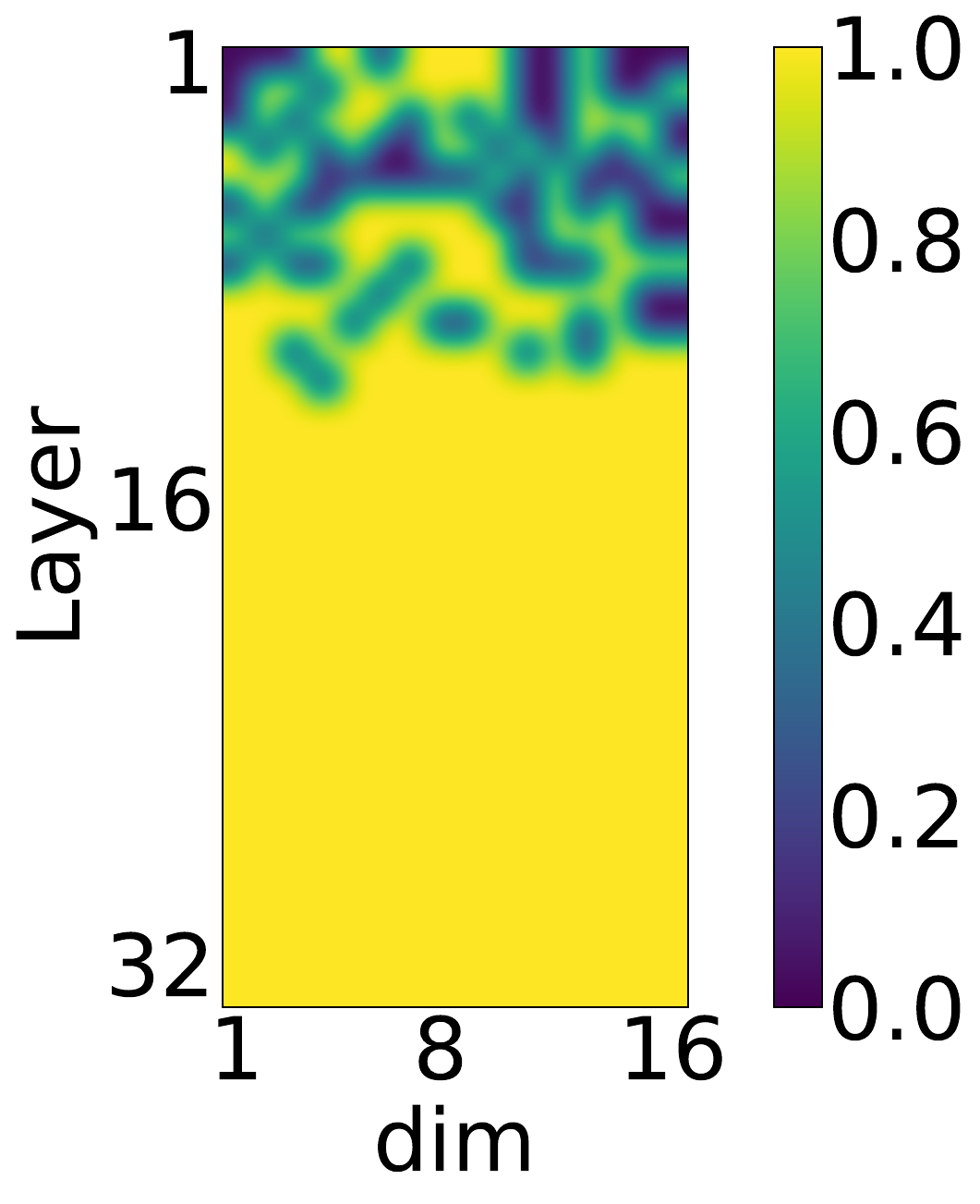}&
\includegraphics[width=0.45\linewidth]{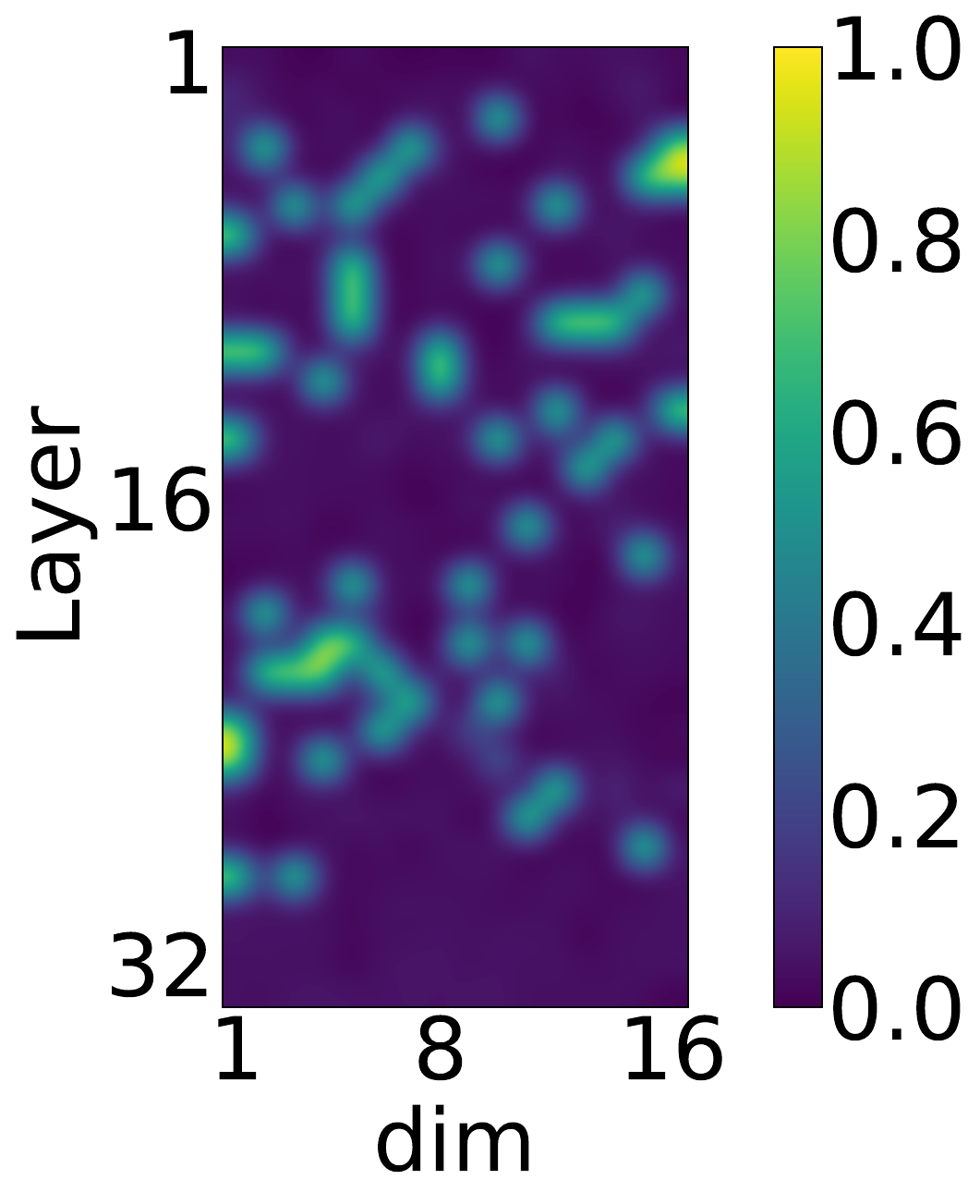}\\ 
{\small a) GCN } & {\small b) GCN-SCT}\\
  \end{tabular}
  \end{center}
  \caption{
  The normalized smoothness -- of each dimension of the feature vectors at a given layer -- for a) GCN and b) GCN-SCT on the Citeseer dataset with 32 layers and 16 hidden dimensions. GCN features become entirely smooth since layer 14, while GCN-SCT controls the smoothness for each feature at any depth. Horizontal and vertical axes represent the index of the feature dimension and the intermediate layer, respectively. 
  }\label{fig:smoothness-GCN-GCNSCT}
\end{figure}

\subsubsection{Other Datasets}
We further compare the performance of different models trained on different datasets using 10-fold cross-validation and fixed $48/32/20\%$ splits following \cite{pei2020geomGCN}. Tables~\ref{table:node:web-wiki} and \ref{table:node:web-wiki-time} compare the accuracy and computational time of GCN and GCNII with and without SCT, using the leaky ReLU activation function, for classifying five heterophilic node classification datasets: Cornell, Texas, Wisconsin, Chameleon, and Squirrel. We exclude EGNN as these heterophilic datasets are not considered in \cite{zhou2021dirichlet}. We report the average test 
accuracy of GCN and GCNII from \cite{chen2020simple}. We tune all other models using a Bayesian meta-learning algorithm to maximize the mean validation accuracy. We report the best test accuracy for each model of depth searched over the set $\{2,4,8,16,32\}$.
SCT can significantly improve the classification accuracy of the baseline models. Table~\ref{table:node:web-wiki} also contrasts the computational time (on Tesla T4 GPUs from Google Colab) per epoch of models that achieve the best test accuracy; the models using SCT can even save computational time to achieve the best accuracy which is because the best accuracy is achieved at a moderate depth (
{Table~\ref{table:node:web-wiki-mean-std-fixed-layers} in 
\ref{appendix:full-sup} lists the mean and standard deviation for the test accuracies on all five datasets.} Table~\ref{table:node:web-wiki-time-fixed-layers} in 
\ref{appendix:full-sup} lists the computational time per epoch for each model of depth 8, showing that using SCT only takes a small amount of computational overhead. 
\begin{table*}[!ht]
\fontsize{9.5}{9.5}\selectfont
\centering
\begin{tabular}{
ccccc}
\specialrule{1.2pt}{1pt}{1pt}
    \textbf{Cornell} & \textbf{Texas} & \textbf{Wisconsin} & \textbf{Chameleon} & \textbf{Squirrel} \\
\hline
$52.70$/
{$\mathbf{55.95}$} 
& $52.16$/
{$\mathbf{62.16}$} 
& $45.88$/
{$\mathbf{54.71}$} 
& $28.18$/
{$\mathbf{38.44}$} 
& $23.96$/
{$\mathbf{35.31}$} 
\\
$74.86$/
{$\mathbf{75.41}$} 
& $69.46$/
{$\textbf{83.34}$} 
& $74.12$/
{$\mathbf{86.08}$} 
& $60.61$/
{$\mathbf{64.52}$} 
& $38.47$/
{$\mathbf{47.51}$} 
\\
    \specialrule{1.2pt}{1pt}{1pt}
    \end{tabular}
    \caption{
    Mean test accuracy 
    for the WebKB and WikipediaNetwork datasets with fixed $48/32/20\%$ splits. First row: GCN/GCN-SCT. Second row: GCNII/GCNII-SCT. 
    (Unit:\%) 
    }
    \label{table:node:web-wiki}
\end{table*}

\begin{table*}[!ht]
\fontsize{9.5}{9.5}\selectfont
\centering
\begin{tabular}{
ccccc}
\specialrule{1.2pt}{1pt}{1pt}
    \textbf{Cornell} & \textbf{Texas} & \textbf{Wisconsin} & \textbf{Chameleon} & \textbf{Squirrel} \\
\hline
$0.7$/$1.8$ & $0.7$/$0.8$ & $0.7$/$0.8$ & $0.6$/$0.7$ & $1.6$/$4.0$ \\
$2.0$/$2.0$ & $3.1$/$2.0$ & $2.0$/$1.5$ & $1.5$/$1.3$ & $5.5$/$3.7$\\
    \specialrule{1.2pt}{1pt}{1pt}
    \end{tabular}
    \caption{
Average computational time per epoch for the WebKB and WikipediaNetwork datasets with fixed $48/32/20\%$ splits. First row: GCN/GCN-SCT. Second row: GCNII/GCNII-SCT. 
    (Unit: $\times10^{-2}$ second) 
    }
    \label{table:node:web-wiki-time}
\end{table*}

\section{Concluding Remarks}\label{sec:conclusion}
In this paper, we have established a geometric characterization of how the ReLU and leaky ReLU activation functions affect the smoothness of the GCN node features. We have further studied the dimension-wise normalized smoothness of the learned node features, showing that activation functions not only smooth node features but also can reduce or preserve the normalized smoothness of the features. Our theoretical findings inform the design of a simple yet effective SCT for GCN. The proposed SCT can change the smoothness, in terms of both normalized and unnormalized smoothness, of the learned node features by GCN. 
%
Our proposed SCT provides provable guarantees for controlling the smoothness of features learned by GCN and related models. A key aspect to establish our theoretical results is demonstrating that, without SCT, the features of the vanilla model tend to be overly smooth; without this condition, SCT cannot ensure performance guarantees.




\appendix

\clearpage
\section{Details of Notations}\label{Appendix-notations}
For two vectors $\vu=(u_1,u_2,\ldots,u_d)$ and $\vv=(v_1,v_2,\ldots,v_d)$, their inner product is defined as
$$
\langle\vu,\vv \rangle = \sum_{i=1}^du_iv_i,
$$
their Hadamard product is defined as
$$
\vu\odot \vv = (u_1v_1,u_2v_2,\ldots,u_dv_d),
$$
and their Kronecker product is defined as
$$
\vu\otimes \vv=  \vu\vv^\top = \begin{pmatrix}
u_1v_1 & u_1v_2 & \ldots &u_1v_d\\
u_2v_1 & u_2v_2 & \ldots &u_2v_d\\
\vdots & \vdots & \ddots &\vdots\\
u_dv_1 & u_dv_2 & \ldots &u_dv_d\\
\end{pmatrix}.
$$
The Kronecker product can be defined for two vectors of different lengths in a similar manner as above.

\section{Proofs in Section~\ref{sec:geometry}}\label{Appendix-proof-geometry}
First, we prove that the two smoothness notions used in \cite{oono2019graph,cai2020note} are two equivalent seminorms, i.e., we prove the following Proposition~\ref{prop:equivalent-norms}:
\begin{proof}[Proof of Proposition~\ref{prop:equivalent-norms}]
The matrix $\mH$ can be decomposed as follows:
$$
\mH = \sum_{i=1}^{n}\mH\ve_i\ve_i^\top,
$$ 
where each $\ve_i$ is the eigenvector of $\mG$ associated with eigenvalue $\lambda_i$. This indicates that
$$
\begin{aligned}
    \mH\Tilde{\Delta} &=\mH(\mI-\mG)\\ &= \sum_{i=1}^{n}\mH\ve_i\ve_i^\top(\mI-\mG) \\
    &= \sum_{i=1}^{n}(\mH\ve_i\ve_i^\top - \mH\ve_i\ve_i^\top\mG) \\
    &=  \sum_{i=1}^{n}(\mH\ve_i\ve_i^\top - \mH\ve_i(\lambda_i\ve_i)^\top) \\
    &=  \sum_{i=1}^{n}(1-\lambda_i)\mH\ve_i\ve_i^\top\\
    &=  \sum_{i=m+1}^{n}(1-\lambda_i)\mH\ve_i\ve_i^\top.
\end{aligned}
$$
Then using the fact that $1-\lambda_i\geq 0$ for each $i$, we obtain 
$$
 \begin{aligned}
    \|\mH\|^2_E &= {\rm Trace}(\mH\Tilde{\Delta}\mH^\top )\\
    &=  {\rm Trace}\Big(\sum_{i=m+1}^{n}(1-\lambda_i)\mH\ve_i\ve_i^\top ( \sum_{j=1}^{n}\mH\ve_j\ve_j^\top)^\top\Big) \\
    &=  {\rm Trace}\Big(\sum_{i=m+1}^{n}\sum_{j=1}^{n}(1-\lambda_i)\mH\ve_i\ve_i^\top \ve_j\ve_j^\top\mH^\top\Big) \\
    &=  {\rm Trace}\Big(\sum_{i=m+1}^{n}(1-\lambda_i)\mH\ve_i\ve_i^\top \ve_i\ve_i^\top\mH^\top\Big) \\
    &=  {\rm Trace}\Big(\sum_{i=m+1}^{n}\sqrt{1-\lambda_i}\mH\ve_i\ve_i^\top \ve_i\ve_i^\top\mH^\top\sqrt{1-\lambda_i}\Big) \\
    &=  {\rm Trace}\Big(\sum_{i=m+1}^{n}\sqrt{1-\lambda_i}\mH\ve_i\ve_i^\top ( \sum_{j=m+1}^{n}\sqrt{1-\lambda_j}\mH\ve_j\ve_j^\top)^\top\Big) \\
    &= \Big\|\sum_{i=m+1}^{n}\sqrt{1-\lambda_i}\mH\ve_i\ve_i^\top\Big\|^2_F.
\end{aligned}
$$
That is, 
$$ 
\begin{aligned}
\|\mH\|_E = \Big\|\sum_{i=m+1}^{n}\sqrt{1-\lambda_i}\mH\ve_i\ve_i^\top\Big\|_F.
\end{aligned}
$$
On the other hand, \eqref{eq:distance-in-F-norm} implies
$$
\begin{aligned}
 \|\mH\|_{\gM^\perp}  = \|\mH_{\gM^\perp}\|_{F} = \Big\|\sum_{i=m+1}^{n}\mH\ve_i\ve_i^\top\Big\|_F.
 \end{aligned}
$$
 
We first show that both  $\|\mH\|_{\gM^\perp}$ and  $\|\mH\|_E $ are seminorms.
Since $\|c\mH\|_F = |c|\cdot\|\mH\|_F$ for any $c\in\sR$, we have $\|c\mH\|_{\gM^\perp} = |c|\cdot\|\mH\|_{\gM^\perp}$ and  $\|c\mH\|_E = |c|\cdot\|\mH\|_E$. Moreover, for any two matrices $\mH^1$ and $\mH^2$ s.t. $\mH= \mH^1 + \mH^2$, we have
$$
\begin{aligned}
    \sum_{i=m+1}^{n}\mH^1\ve_i\ve_i^\top + \sum_{i=m+1}^{n}\mH^2\ve_i\ve_i^\top &= \sum_{i=m+1}^{n}
    {\mH}\ve_i\ve_i^\top, \\
    \sum_{i=m+1}^{n}\sqrt{1-\lambda_i}\mH^1\ve_i\ve_i^\top + \sum_{i=m+1}^{n}\sqrt{1-\lambda_i}\mH^2\ve_i\ve_i^\top &= \sum_{i=m+1}^{n}\sqrt{1-\lambda_i}
    {\mH}\ve_i\ve_i^\top.
\end{aligned}
$$
Then the triangle inequality of $\|\cdot\|_F$ implies that of $\|\mH\|_{\gM^\perp}$ and  $\|\mH\|_E$, respectively.
 
Now since $0<1-\lambda_{m+1} \leq 1-\lambda_i \leq 2$ for any $i=m+1,\ldots, n$, we may take $\alpha = \sqrt{1-\lambda_{m+1}}$ and $\beta = \sqrt{2}$. Then
$$
\begin{aligned}
    \alpha  \|\mH\|_{\gM^\perp}  &= \Big\|\alpha\sum_{i=m+1}^{n}\mH\ve_i\ve_i^\top\Big\|_F\\
    &\leq \Big\|\sum_{i=m+1}^{n}\sqrt{1-\lambda_i}\mH\ve_i\ve_i^\top\Big\|_F\\
    &\leq  \Big\|\beta\sum_{i=m+1}^{n}\mH\ve_i\ve_i^\top\Big\|_F\\
    &= \beta \|\mH\|_{\gM^\perp}.
\end{aligned}
$$
The result thus follows from $\|\mH\|_E =\Big\|\sum_{i=m+1}^{n}\sqrt{1-\lambda_i}\mH\ve_i\ve_i^\top\Big\|_F$.
\end{proof}

\subsection{ReLU}
We present a crucial {tool} to characterize how ReLU affects its input.
\begin{lemma}\label{lemma:positive-negative-part}
Let ${ \mZ\in\sR^{d\times n}}$, and let ${ \mZ^+ =\max(\mZ, 0)}$ and ${ \mZ^-=\max(-\mZ, 0)}$ be the positive and negative parts of $\mZ$. Then (1) ${ \mZ^+,\mZ^-}$ are (component-wise) nonnegative and ${ \mZ=\mZ^+-\mZ^-}$ and (2) ${ \langle\mZ^+, \mZ^-\rangle_F=0}$.
\end{lemma}

\begin{proof}[Proof of Lemma~\ref{lemma:positive-negative-part}] 
Notice that for any $a\in\sR$, we have
$$
    \begin{aligned}
        \max(a, 0) = 
        \begin{cases}
              a & \text{ if } a\geq0 \\
            0 &\text{otherwise}
        \end{cases}
    \end{aligned}. 
    $$
and 
$$
    \begin{aligned}
                \max(-a, 0) = 
                \begin{cases}
              0 & \text{ if } a\geq0 \\
            -a &\text{otherwise}
            \end{cases}
    \end{aligned}. 
    $$
    This implies that $a = \max(a, 0)-\max(-a, 0)$ and $\max(a, 0)\cdot \max(-a, 0) = 0$.
    
    Let $Z_{ij}$ be the $(i,j)^{th}$ entry of $\mZ$. Then $\mZ = \mZ^+ - \mZ^-$ follows from $Z_{ij} = \max(Z_{ij}, 0)-\max(-Z_{ij}, 0)$.  Also, one can deduce that
    $$
    \begin{aligned}
    \langle \mZ^+, \mZ^- \rangle_F &= {\rm Trace}((\mZ^+)^\top \mZ^-)\\
    &= \sum_{i=1}^d \sum_{j=1}^j \max(Z_{ij},0)\max(-Z_{ij},0)
    = 0.
    \end{aligned}
    $$
\end{proof}

Before proving Proposition~\ref{prop:relu-smoothness-geometric}, we notice the following relation between $\mZ$ and $\mH$. 
\begin{lemma}\label{prop:relu-circle-condition}
Given $\mZ\in\sR^{d\times n}$, let $\mH = \sigma(\mZ)$ with $\sigma$ being the ReLU activation function, then $\mH$ lies on the high-dimensional sphere, in $\|\cdot\|_F$ norm, that is centered at ${\mZ}/{2}$ and with radius $\|{\mZ}/{2}\|_F$. That is, $\mH$ and $\mZ$ satisfy the following equation:
\begin{equation}\label{eq:sphere}
\begin{aligned}
\Big\| \mH - \frac{\mZ}{2} \Big\|^2_F  = \Big\|\frac{\mZ}{2}\Big\|^2_F.
\end{aligned}
\end{equation}
\end{lemma}
\begin{proof}[Proof of Lemma~\ref{prop:relu-circle-condition}]
  We observe that $\mH = \sigma(\mZ) = \max(\mZ, 0) = \mZ^+$ is the positive part of $\mZ$. Then we have 
$$
\begin{aligned}
\langle \mH, \mZ \rangle_F = \langle \mH, \mZ^+ - \mZ^- \rangle_F = \langle \mH, \mZ^+ \rangle_F -  \langle \mH, \mZ^- \rangle_F = \langle \mH, \mH \rangle_F,
\end{aligned}
$$
where we have used $\mZ = \mZ^+ - \mZ^-$ and $\langle \mH, \mZ^- \rangle_F= \langle \mZ^+, \mZ^- \rangle_F=0$ from Lemma \ref{lemma:positive-negative-part}.

Therefore, one can deduce the desired result as follows:
$$
\begin{aligned}
    \langle \mH, \mH \rangle_F - \langle \mH, \mZ \rangle_F = 0 
    \Rightarrow& \| \mH\|^2_F - 2\Big\langle \mH, \frac{\mZ}{2} \Big\rangle_F + \Big\|\frac{\mZ}{2}\Big\|^2_F  = \Big\|\frac{\mZ}{2}\Big\|^2_F \\
    \Rightarrow& \Big\| \mH - \frac{\mZ}{2} \Big\|^2_F  = \Big\|\frac{\mZ}{2}\Big\|^2_F.  \label{prop:circle-condition-2norm}
\end{aligned}
$$
\end{proof}

Applying $\|\mH\|^2_F = \|\mH_\gM + \mH_{\gM^\perp} \|_F^2 = \|\mH_\gM\|^2_F+\|\mH_{\gM^\perp}\|^2_F $, to both $\frac{\mZ}{2}$ and $\mH - \frac{\mZ}{2}$, we obtain
$$
\begin{aligned}
\Big\|\frac{\mZ}{2}\Big\|^2_F
    =\Big\|\frac{\mZ_{\gM^\perp}}{2}\Big\|^2_F+\Big\|\frac{\mZ_{\gM}}{2}\Big\|^2_F,
\end{aligned}
$$
and
$$
\begin{aligned}
\Big\|\mH-\frac{\mZ}{2}\Big\|^2_F
    = \Big\| \mH_{\gM^\perp} - \frac{\mZ_{\gM^\perp}}{2} \Big\|^2_F + \Big\| \mH_{\gM} - \frac{\mZ_{\gM}}{2} \Big\|^2_F. 
\end{aligned}
$$
Then \eqref{eq:sphere} becomes
\begin{equation}
\label{eq:inner-product-condition-relu-00}
    \begin{aligned}
    \Big\|\frac{\mZ_{\gM^\perp}}{2}\Big\|^2_F - \Big\| \mH_{\gM^\perp} - \frac{\mZ_{\gM^\perp}}{2} \Big\|^2_F
    =
    \Big\| \mH_{\gM} - \frac{\mZ_{\gM}}{2} \Big\|^2_F  -\Big\|\frac{\mZ_{\gM}}{2}\Big\|^2_F 
    \end{aligned}
\end{equation}
By direct calculation, we have
\begin{equation}
\label{eq:inner-product-condition-relu-01}
    \begin{aligned}
        \Big\| \mH_{\gM} - \frac{\mZ_{\gM}}{2} \Big\|^2_F  -\Big\|\frac{\mZ_{\gM}}{2}\Big\|^2_F
        &=
        \langle \mH_{\gM}, \mH_{\gM} \rangle_F - 2 \Big\langle \mH_{\gM}, \frac{\mZ_{\gM}}{2} \Big\rangle_F\\
        &= 
        \langle \mH_{\gM}, \mH_{\gM}- \mZ_{\gM} \rangle_F.
    \end{aligned}
\end{equation}
Combining \eqref{eq:inner-product-condition-relu-00} and \eqref{eq:inner-product-condition-relu-01}, we obtain the following result
\begin{lemma}\label{prop:relu-smoothness-quantity}
For any ${ \mZ=\mZ_\gM+\mZ_{\gM^\perp}}$, 
let 
$$
{ \mH=\sigma(\mZ) = \mH_\gM + \mH_{\gM^\perp}},
$$ 
then we have
\begin{equation*} 
\begin{aligned}
\Big\|\frac{\mZ_{\gM^\perp}}{2}\Big\|^2_F - \Big\| \mH_{\gM^\perp}- \frac{\mZ_{\gM^\perp}}{2} \Big\|^2_F
=\langle \mZ^+_{\gM}, \mZ^-_{\gM} \rangle_F. 
\end{aligned}
\end{equation*}
where $\mZ^+_{\gM} = \sum_{i=1}^{m}\mZ^+\ve_i\ve_i^\top, \mZ^-_{\gM} = \sum_{i=1}^{m}\mZ^-\ve_i\ve_i^\top$.
\end{lemma}

\begin{proof}[Proof of Lemma~\ref{prop:relu-smoothness-quantity}]
Recall that $\mH = \sigma(\mZ) = \max(\mZ, 0) = \mZ^+$. Also, $\mZ = \mZ^+ - \mZ^-$ implies $\mZ_{\gM} = \mZ^+_{\gM} - \mZ^-_{\gM} = \mH^+_{\gM} - \mZ^-_{\gM}$. Therefore, we see that
$$
\begin{aligned}
\langle \mH_{\gM}, \mH_{\gM}- \mZ_{\gM} \rangle_F = \langle \mZ^+_{\gM}, \mZ^-_{\gM} \rangle_F.
\end{aligned}
$$
\end{proof}

By using the fact that $\langle \mZ^+_{\gM}, \mZ^-_{\gM} \rangle_F \geq 0$ in Lemma~\ref{prop:relu-smoothness-quantity}, we reveal a geometric relation between $\mZ$ and $\mH$ mentioned in Proposition~\ref{prop:relu-smoothness-geometric}. 

\begin{proof}[Proof of Proposition~\ref{prop:relu-smoothness-geometric}]
Since $\mZ^+, \mZ^- \geq 0$ are nonnegative and all the eigenvectors $\ve_i$ are also nonnegative, we see that $\mZ^+_{\gM} = \sum_{i=1}^m \mZ^+\ve_i\ve_i^\top$ and $\mZ^-_{\gM} = \sum_{i=1}^m \mZ^-\ve_i\ve_i^\top$ are nonnegative. This indicates that
$$
\begin{aligned}
     \langle \mZ^+_{\gM}, \mZ^-_{\gM} \rangle_F = {\rm Trace}\Big(\mZ^+_{\gM}(\mZ^-_{\gM})^\top \Big) \geq 0.
\end{aligned}
$$
Then according to {Lemma}~\ref{prop:relu-smoothness-quantity}, we obtain
$$
\begin{aligned}
\Big\|\frac{\mZ_{\gM^\perp}}{2}\Big\|^2_F - \Big\| \mH_{\gM^\perp} - \frac{\mZ_{\gM^\perp}}{2} \Big\|^2_F =  \langle \mZ^+_{\gM}, \mZ^-_{\gM} \rangle_F \geq 0.
\end{aligned}
$$
So we have 
$$
\begin{aligned}
\Big\| \mH_{\gM^\perp} - \frac{\mZ_{\gM^\perp}}{2} \Big\|_F 
&= \sqrt{\Big\|\frac{\mZ_{\gM^\perp}}{2}\Big\|^2_F - \langle \mZ^+_{\gM}, \mZ^-_{\gM} \rangle_F}\\
& =
 \sqrt{\Big\|\frac{\mZ_{\gM^\perp}}{2}\Big\|^2_F - \langle \mH_{\gM}, \mH_{\gM}- \mZ_{\gM} \rangle_F},
 \end{aligned}
 $$
which shows that $\mH_{\gM^\perp}$ lies on the high-dimensional sphere that we have claimed.
Furthermore, we conclude that
\begin{equation}
\label{prop:circle-condition-inequality}
    \begin{aligned}
    0\leq \Big\| \mH_{\gM^\perp} - \frac{\mZ_{\gM^\perp}}{2} \Big\|_F \leq \Big\|\frac{\mZ_{\gM^\perp}}{2}\Big\|_F.
\end{aligned}
\end{equation}
This demonstrates that $\mH_{\gM^\perp}$ lies on the high-dimensional sphere we have stated. 

Since the sphere $\Big\| \mH_{\gM^\perp} - \frac{\mZ_{\gM^\perp}}{2} \Big\|^2_F = \Big\|\frac{\mZ_{\gM^\perp}}{2}\Big\|^2_F$ passes through the origin, the distance of any $\mH_{\gM^\perp}$ to the origin must be no greater than the diameter of this sphere, i.e., $\|\mH_{\gM^\perp}\|_F\leq \|\mZ_{\gM^\perp}\|_F$.
Also, this can be derived from the following inequality:
$$
\begin{aligned}
\|\mH_{\gM^\perp}\|_F - \Big\|\frac{\mZ_{\gM^\perp}}{2}\Big\|_F \leq \Big\| \mH_{\gM^\perp} - \frac{\mZ_{\gM^\perp}}{2} \Big\|_F \leq \Big\|\frac{\mZ_{\gM^\perp}}{2}\Big\|_F.
\end{aligned}
$$
One can see that the maximal smoothness $\|\mH_{\gM^\perp}\|_F =  \|\mZ_{\gM^\perp}\|_F$ is attained when $\mH_{\gM^\perp} = \mZ_{\gM^\perp}$, the intersection of the surface and the line passing through the center and the origin.

After all, we complete the proof by using the fact that $\|\mZ_{\gM^\perp}\|_F = \|\mZ\|_{\gM^\perp}$ for any matrix $\mZ$, which implies $ \|\mH\|_{\gM^\perp} = \|\mH_{\gM^\perp}\|_F \leq \|\mZ_{\gM^\perp}\|_F =  \|\mZ\|_{\gM^\perp}.$

\end{proof}

\subsection{Leaky ReLU}
For the leaky ReLU activation function, we have the following Lemma:
\begin{lemma}\label{prop:leaky-circle-condition}
If $\mH = \sigma_a(\mZ)$ with $\sigma_a$ being the leaky ReLU activation function, then $\mH$ lies on the high-dimensional sphere centered at ${(1+a)\mZ}/{2}$ with radius $\|{(1-a)\mZ}/{2}\|_F$. 
\end{lemma}

\begin{proof}[Proof of Lemma~\ref{prop:leaky-circle-condition}]
Notice that 
$$
\begin{aligned}
\mH = \sigma_a(\mZ) = \mZ^+ - a\mZ^-.
\end{aligned}
$$
Then 
$\mH- \mZ = (1-a)\mZ^-$ and $\mH - a\mZ = (1-a)\mZ^+$. Using $\langle \mZ^- , \mZ^+ \rangle _F=0$, we have 
\begin{equation*}
    \begin{aligned}
        &\langle \mH- \mZ , \mH - a\mZ \rangle_F = 0 \\
        \Rightarrow & \|\mH\|^2_F -2\Big\langle \mH, \frac{(1+a)\mZ}{2} \Big\rangle_F + a\|\mZ\|^2_F= 0 \\
        \Rightarrow & \|\mH\|^2_F -2\Big\langle \mH, \frac{(1+a)\mZ}{2} \Big\rangle_F = - a\|\mZ\|^2_F \\
        \Rightarrow & \Big\| \mH - \frac{(1+a)}{2}\mZ\Big\|^2_F = \Big\| \frac{(1+a)}{2}\mZ\Big\|^2_F - a\|\mZ\|^2_F = \Big\| \frac{(1-a)}{2}\mZ\Big\|^2_F.
    \end{aligned}
\end{equation*}
\end{proof}

Moreover, we notice that 
\begin{lemma}\label{prop:leaky-relu-smoothness-quantity}
For any 
$ \mZ=\mZ_\gM+\mZ_{\gM^\perp}$, let $ \mH=\sigma_a(\mZ)=\mH_\gM+\mH_{\gM^\perp}$, then
\begin{equation*}
\begin{aligned}
\Big\|\frac{(1-a)}{2}\mZ_{\gM^\perp}\Big\|^2_F - \Big\| \mH_{\gM^\perp} - \frac{(1+a)}{2}\mZ_{\gM^\perp} \Big\|^2_F
= (1-a)^2\langle \mZ^+_{\gM} , \mZ^-_{\gM} \rangle_F
\end{aligned}
\end{equation*}
\end{lemma}

\begin{proof}[Proof of Lemma~\ref{prop:leaky-relu-smoothness-quantity}]
Similar to the proof of Lemma~\ref{prop:relu-smoothness-quantity}, the orthogonal decomposition implies that
$$
\begin{aligned}   
    &\Big\|\frac{(1-a)}{2}\mZ_{\gM^\perp}\Big\|^2_F - \Big\| \mH_{\gM^\perp} - \frac{(1+a)}{2}\mZ_{\gM^\perp} \Big\|^2_F\\
    =&   \Big\| \mH_{\gM} - \frac{(1+a)}{2}\mZ_{\gM} \Big\|^2_F - \Big\|\frac{(1-a)}{2}\mZ_\gM\Big\|^2_F \\
    =& \langle \mH_{\gM}- \mZ_{\gM} , \mH_{\gM}- a\mZ_{\gM} \rangle_F \\
    =& \langle (1-a)\mZ^-_{\gM} , (1-a)\mZ^+_{\gM} \rangle_F \\
    =& (1-a)^2\langle \mZ^-_{\gM} , \mZ^+_{\gM} \rangle_F.
\end{aligned}
$$
\end{proof}

\begin{proof}[Proof of Proposition~\ref{prop:leaky-relu-smoothness-geometric}]
Similar to the proof of Proposition \ref{prop:relu-smoothness-geometric}, we apply $\langle \mZ^-_{\gM} , \mZ^+_{\gM} \rangle_F \geq 0 $ to Lemma~\ref{prop:leaky-relu-smoothness-quantity} and hence obtain the geometric condition as follows:
$$
\begin{aligned}
&\Big\| \mH_{\gM^\perp} - \frac{(1+a)}{2}\mZ_{\gM^\perp} \Big\|_F \\
=& \sqrt{\Big\|\frac{(1-a)}{2}\mZ_{\gM^\perp}\Big\|^2_F - \langle \mH_{\gM}- \mZ_{\gM} , \mH_{\gM}- a\mZ_{\gM} \rangle_F}.
\end{aligned}
$$
Then we have the following inequality:
$$
\begin{aligned}
0 \leq \Big\| \mH_{\gM^\perp} - \frac{(1+a)}{2}\mZ_{\gM^\perp} \Big\|_F \leq  \Big\|\frac{(1-a)}{2}\mZ_{\gM^\perp}\Big\|_F.
\end{aligned}
$$
Moreover, we deduce that
$$
\begin{aligned}
\Bigg|\|\mH_{\gM^\perp}\|_F - \Big\|\frac{(1+a)}{2}\mZ_{\gM^\perp}\Big\|_F\Bigg| &\leq\Big\| \mH_{\gM^\perp} - \frac{(1+a)}{2}\mZ_{\gM^\perp} \Big\|_F\\
&\leq  \Big\|\frac{(1-a)}{2}\mZ_{\gM^\perp}\Big\|_F.
\end{aligned}
$$
and hence
$$
\begin{aligned}
-\Big\|\frac{(1-a)}{2}\mZ_{\gM^\perp}\Big\|_F &\leq \|\mH_{\gM^\perp}\|_F - \Big\|\frac{(1+a)}{2}\mZ_{\gM^\perp}\Big\|_F\\
&\leq  \Big\|\frac{(1-a)}{2}\mZ_{\gM^\perp}\Big\|_F.
\end{aligned}
$$
Therefore, we obtain $a\|\mZ_{\gM^\perp}\|_F\leq \|\mH_{\gM^\perp}\|_F\leq \|\mZ_{\gM^\perp}\|_F$.
(Remark that $\mH_{\gM^\perp}$ achieves its maximal norm when it is equal to $\mZ_{\gM^\perp}$, the intersection of the surface and the line passing through the center and the origin.)

By using the fact that $\|\mZ_{\gM^\perp}\|_F = \|\mZ\|_{\gM^\perp}$ for any matrix $\mZ$, we conclude that $ a\|\mZ\|_{\gM^\perp}\leq \|\mH\|_{\gM^\perp} \leq  \|\mZ\|_{\gM^\perp}.$
\end{proof}

\section{Proofs in Section~\ref{sec:control}}\label{sec:sct}

Throughout this section, we assume that 
$\vz_{\gM^\perp} \neq {\bf 0}$.

\begin{proof}[Proof of Proposition~\ref{prop:smoothness-control-relu}]
Recall that $\ve = \Tilde{\mD}^{\frac{1}{2}}\vu_n/c$ has only positive entries where $\Tilde{\mD}$ is the augmented degree matrix and $\vu_n = [1, \ldots, 1]^\top \in \sR^n$ and $c = \|\Tilde{\mD}^{\frac{1}{2}}\vu_n\|$. Let $d_i$ be the $i^{th}$ diagonal entry of $\Tilde{\mD}$. Then we have 
$
\ve = [\sqrt{d_1}/c, \sqrt{d_2}/c, \ldots, \sqrt{d_n}/c]^\top$
and $c = \sqrt{\sum^n_{i=1}d_i}$.

Note that $\vz(\alpha) = \vz - \alpha\ve = \vz -\frac{\alpha}{c} \Tilde{\mD}^{\frac{1}{2}}\vu_n = \Tilde{\mD}^{\frac{1}{2}}(\Tilde{\mD}^{-\frac{1}{2}}\vz - \frac{\alpha}{c}\vu_n) = \Tilde{\mD}^{\frac{1}{2}}(\vx - \frac{\alpha}{c}\vu_n)$, where we assume $\vx \coloneqq \Tilde{\mD}^{-\frac{1}{2}}\vz$. Then we observe that when $\sigma$ is the ReLU activation function, 
$$
\begin{aligned}
\vh(\alpha) = \sigma(\vz(\alpha))
= \sigma\Big(\Tilde{\mD}^{\frac{1}{2}}(\vx - \frac{\alpha}{c}\vu_n)\Big)
= \Tilde{\mD}^{\frac{1}{2}}\sigma\Big(\vx - \frac{\alpha}{c}\vu_n\Big),
\end{aligned}
$$
and hence 
$$
\begin{aligned}
\langle \vh(\alpha), \ve \rangle &= \Big\langle \Tilde{\mD}^{\frac{1}{2}}\sigma\Big(\vx-\frac{\alpha}{c}\vu_n\Big),\ve\Big\rangle\\
&=\Big\langle \sigma\Big(\vx - \frac{\alpha}{c}\vu_n\Big), \Tilde{\mD}^{\frac{1}{2}}\ve\Big\rangle\\
&= \Big\langle \sigma\Big(\vx-\frac{\alpha}{c}\vu_n\Big),\Tilde{\mD}\vu_n\Big\rangle.
\end{aligned}
$$
We may now assume $\vx = [ x_1, \ldots, x_n]^\top$ is well-ordered s.t. $x_1 \geq x_2 \geq \ldots \geq x_n$. Indeed, there is a collection of indices $\{k_1,...,k_l\}$ such that 
$$
\begin{aligned}
    &x_1 = \ldots,x_{k_1} \text{ and } x_{k_1} >x_{k_1+1}, \\
    &x_{k_{j-1}+1} = \ldots = x_{k_{j}} \text{ and } x_{k_{j}} >x_{k_{j}+1} \text{ for any } j = 2,\ldots, l-1,\\
    &x_{k_{l-1}+1} = \ldots =x_{k_{l}} \text{ and } k_l = n. 
\end{aligned}
$$
That is, $x_1 = x_2 =\ldots = x_{k_1} >x_{k_1+1} = \ldots = x_{k_2} >x_{k_2+1}  = \ldots = x_{k_3} >  x_{k_3 + 1} \ldots$. 

We first restrict the domain of $\alpha$ s.t. $\vh(\alpha) \neq 0$. Note that we have 
$$
\begin{aligned}
\vh(\alpha)  =0
\Leftrightarrow&  
    \sigma\Big(\vx - \frac{\alpha}{c}\vu_n\Big)=0\\
    \Leftrightarrow&   x_i - \frac{\alpha}{c} \leq 0 \text{ for }i = 1,\ldots, n\\     
    \Leftrightarrow &  x_1 - \frac{\alpha}{c} \leq 0  \\
    \Leftrightarrow &  \alpha \geq c x_1.
\end{aligned}
$$
So we will study the smoothness $s(\vh(\alpha))$ when  $\alpha < c x_1$.

Let $\epsilon>0$ and consider $\alpha = c(x_1-\epsilon)$. When $\epsilon \leq x_1 - x_{k_1+1} = x_1 - x_{k_2}$, we see that 
$$
\vx -  \frac{\alpha}{c}\vu_n = [\epsilon, \ldots, \epsilon, \epsilon-(x_1-x_{k_1+1}), \ldots, \epsilon-(x_1-x_{n})]^\top,
$$
where only the first $k_1$ entries are positive since $x_1-x_{i} \geq \epsilon$ for any $i\geq k_1+1$.
Therefore,
$$
\begin{aligned}
\vh(\alpha) &= \Tilde{\mD}^{\frac{1}{2}}\sigma\Big(\vx-\frac{\alpha}{c}\vu_n\Big)\\
&= \Tilde{\mD}^{\frac{1}{2}}[\epsilon,\ldots, \epsilon,0,\ldots,0]^\top\\
&=[\epsilon\sqrt{d_1}, \ldots,\epsilon\sqrt{d_{k_1}},0,\ldots,0]^\top.
\end{aligned}
$$
and hence we can compute that
$
\begin{aligned}
\|\vh(\alpha)\| = \epsilon\sqrt{\sum^{k_1}_{i=1} d_i}.
\end{aligned}
$
Also, we have
$$
\begin{aligned}
\|\vh(\alpha)\|_\gM &= |\langle \vh(\alpha), \ve \rangle|\\
&= [\epsilon\sqrt{d_1}, \ldots, \epsilon\sqrt{d_{k_1}}, 0, \ldots, 0]^\top[\sqrt{d_1}/c, \sqrt{d_2}/c, \ldots, \sqrt{d_n}/c]\\
&=  \frac{\epsilon}{c} \sum^{k_1}_{i=1} d_i.
\end{aligned}
$$
Then we obtain the smoothness $s(\vh(\alpha))$ as follows
$$
\begin{aligned}
s(\vh(\alpha)) = \frac{\|\vh(\alpha)\|_\gM}{\|\vh(\alpha)\|} = \frac{\frac{\epsilon}{c} \sum^{k_1}_{i=1} d_i}{ \epsilon\sqrt{\sum^{k_1}_{i=1} d_i}} = \frac{\sqrt{\sum^{k_1}_{i=1} d_i}}{c} = \frac{K_1}{c}<1,
\end{aligned}
$$
where $K_1\coloneqq\sqrt{\sum^{k_1}_{i=1} d_i}$. Similarly, we may denote $\sqrt{\sum^{k_j}_{i=k_{j-1}+1} d_i}$ by $K_j$ for $j = 2,\ldots, l$.

Now we are going to show that the smoothness $s(\vh(\alpha))$ is increasing as $\alpha$ gets smaller whenever $\alpha<cx_1$,
implying $\frac{K_1}{c}$ is the minimum of the smoothness $s(\vh(\alpha))$.
Remember that we are considering  $\alpha = c(x_1-\epsilon)$ and we have studied the case when $0< \epsilon \leq x_1 - x_{k_1+1} = x_1 - x_{k_2}$.

Let $\delta_j \coloneqq x_1 - x_{k_{j}}$ for $1\leq j\leq l$. Clearly, we have $\delta_1 = 0$ and $\delta_j < \delta_{j+1}$ for $1\leq j\leq l-1$.
Fix a $j'\in\{2,\ldots, l-1\}$, we see that when $\delta_{j'} <\epsilon \leq x_1 - x_{k_{j'}+1}$,  
$$
\begin{aligned}
&\vx - \frac{\alpha}{c}\vu_n \\
&= \Big[\epsilon -\delta_1, \ldots, \epsilon-\delta_1, \epsilon-\delta_2,\\
&\ \ \ \ \ \ldots, \epsilon-\delta_2, \epsilon-\delta_3, \ldots,  \epsilon-\delta_{j'} , \epsilon-(x_1-x_{k_{j'}+1}),\ldots, \epsilon-(x_1-x_{n})\Big]^\top,
\end{aligned}
$$
where we have $\epsilon - \delta_j>0$ for $2\leq j\leq j'$ and $\epsilon - (x_1-x_{i}) \leq 0 $ for any $i\geq k_{j'}+1$.
Consequently,
$$
\begin{aligned}
    \vh(\alpha) &= \Tilde{\mD}^{\frac{1}{2}}\sigma(\vx -  \frac{\alpha}{c}\vu_n) \\
&= [(\epsilon-\delta_1)\sqrt{d_1}, \ldots, (\epsilon-\delta_1)\sqrt{d_{k_1}}, (\epsilon-\delta_2)\sqrt{d_{k_1+1}}, \ldots, (\epsilon-\delta_2)\sqrt{d_{k_2}},\\
&\qquad (\epsilon-\delta_3)\sqrt{d_{k_2+1}}, \ldots, (\epsilon-\delta_{j'})\sqrt{d_{k_{j'}}}, 0, \ldots, 0]^\top.
\end{aligned}
$$
Then we can compute
$$
\begin{aligned}
\|\vh(\alpha)\| 
&= \sqrt{\sum^{j'}_{j = 1}\sum^{k_{j}}_{i=k_{j-1}+1} d_i (\epsilon-\delta_j)^2 }\\
&= \sqrt{\sum^{j'}_{j = 1}K_j^2(\epsilon-\delta_j)^2},
\end{aligned}
$$
where we set $k_0 \coloneqq 0$ for simplicity and $K_j = \sqrt{\sum^{k_{j}}_{i=k_{j-1}+1} d_i}$ for $j = 1,\ldots, j'$. Also, we have
$$
\begin{aligned}
\|\vh(\alpha)\|_\gM = |\langle \vh(\alpha), \ve \rangle| 
&=  \sum^{j'}_{j = 1} \sum^{k_{j}}_{i=k_{j-1}+1} \frac{d_i(\epsilon-\delta_j)}{c}\\
&= \frac{1}{c}\sum^{j'}_{j = 1}K_j^2(\epsilon-\delta_j).
\end{aligned}
$$

A careful calculation shows that $\frac{\partial}{\partial \epsilon}s(\vh(\alpha)) > 0 $ whenever $\delta_{j'} <\epsilon \leq x_1 - x_{k_{j'}+1}$ which implies that $s(\vh(\alpha))$ is increasing as $\epsilon$ increases. Indeed, we have 
$$
\begin{aligned}
    &\frac{\partial}{\partial \epsilon}s(\vh(\alpha)) \\
    =& \frac{\partial}{\partial \epsilon}\Bigg(\frac{\sum^{j'}_{j = 1}K_j^2(\epsilon-\delta_j) }{c \sqrt{\sum^{j'}_{j = 1}K_j^2(\epsilon-\delta_j)^2}}\Bigg) \\
=& \frac{\Big(\frac{\partial}{\partial \epsilon} \sum^{j'}_{j = 1}K_j^2(\epsilon-\delta_j)\Big)
    \sqrt{\sum^{j'}_{j = 1}K_j^2(\epsilon-\delta_j)^2}  
   }
    {c\sum^{j'}_{j = 1}K_j^2(\epsilon-\delta_j)^2} -\\
& \ \ - \frac{\sum^{j'}_{j = 1}K_j^2(\epsilon-\delta_j)\Big(\frac{\partial}{\partial \epsilon}\sqrt{\sum^{j'}_{j = 1}K_j^2(\epsilon-\delta_j)^2}\Big)}
    {c\sum^{j'}_{j = 1}K_j^2(\epsilon-\delta_j)^2}\\
    \\
 =& \frac{\Big(\sum^{j'}_{j = 1}K_j^2\Big)\sqrt{\sum^{j'}_{j = 1}K_j^2(\epsilon-\delta_j)^2}  -  \sum^{j'}_{j = 1}K_j^2(\epsilon-\delta_j)\Big(\frac{\frac{\partial}{\partial \epsilon} \sum^{j'}_{j = 1}K_j^2(\epsilon-\delta_j)^2}{2\sqrt{\sum^{j'}_{j = 1}K_j^2(\epsilon-\delta_j)^2}}\Big) }{c\sum^{j'}_{j = 1}K_j^2(\epsilon-\delta_j)^2} \\
 =& \frac{\Big(\sum^{j'}_{j = 1}K_j^2\Big)\sum^{j'}_{j = 1}K_j^2(\epsilon-\delta_j)^2 -  \sum^{j'}_{j = 1}K_j^2(\epsilon-\delta_j)\Big(\sum^{j'}_{j = 1}K_j^2(\epsilon-\delta_j)\Big) }{c\sum^{j'}_{j = 1}K_j^2(\epsilon-\delta_j)^2\sqrt{\sum^{j'}_{j = 1}K_j^2(\epsilon-\delta_j)^2}}.
    \end{aligned}
$$

Then to show that $\frac{\partial}{\partial \epsilon}s(\vh(\alpha)) > 0 $, it suffices to show that the numerator is positive, i.e.
$$
\begin{aligned}
\Big(\sum^{j'}_{j = 1}K_j^2\Big)\sum^{j'}_{j = 1}K_j^2(\epsilon-\delta_j)^2 -  \Big(\sum^{j'}_{j = 1}K_j^2(\epsilon-\delta_j)\Big)^2 >0,
\end{aligned}
$$
since the denominator $c\sum^{j'}_{j = 1}K_j^2(\epsilon-\delta_j)^2\sqrt{\sum^{j'}_{j = 1}K_j^2(\epsilon-\delta_j)^2}>0$ is always positive.
In fact, this follows from the Cauchy inequality $\|\vv\|\|\vu\|\geq\langle\vv,\vu\rangle$, where we set
$$
\begin{aligned}
    \vv &\coloneqq [K_1, K_2, \ldots, K_{J'}]^\top, \\
    \vu &\coloneqq [K_1(\epsilon-\delta_1), K_2(\epsilon-\delta_2), \ldots, K_{j'}(\epsilon-\delta_{j'})]^\top.
\end{aligned}
$$
Moreover, equality happens only when $\vv$ is parallel to $\vu$. This is, however, impossible since $\epsilon-\delta_j>\epsilon-\delta_{j+1}$ for any $j=1,\ldots, j'-1$ and each $K_j$ is positive.

So we see that  $s(\vh(\alpha))$ is increasing as $\epsilon$ increases whenever $0 <\epsilon$, and hence the smoothness  $s(\vh(\alpha))$ is increasing as $\alpha$ decreases whenever $c x_n\leq \alpha< cx_1$. 

For the case $j' = l$ where $\delta_l = x_1 - x_n < \epsilon$, we have $x_n - \alpha/c = x_n - (x_1-\epsilon)  = \epsilon - (x_1 - x_n) > 0$, implying $\alpha< c x_n $ and $\vh(\alpha) = \vz(\alpha)$. We have shown that the smoothness is increasing as $\alpha$ is going far from $\langle\vz, \ve\rangle$; in particular, when $\alpha<\langle\vz, \ve\rangle$ and $\alpha$ is decreasing. One can check that
$$
\begin{aligned}
c x_n = \frac{\sum^n_{i=1} d_ix_n}{c} &= \bigg\langle x_n \vu_n, \frac{\Tilde{\mD}\vu_n}{c}\bigg\rangle\\
&\leq \bigg\langle \vx, \frac{\Tilde{\mD}\vu_n}{c}\bigg\rangle\\ 
&= \bigg\langle \Tilde{\mD}^{\frac{1}{2}}\vx, \frac{\Tilde{\mD}^{\frac{1}{2}}\vu_n}{c}\bigg\rangle\\
&= \langle\vz, \ve\rangle,
\end{aligned}
$$
which means the smoothness is increasing as $\alpha$ decreases whenever $\alpha<c x_n$.

We conclude that the smoothness increases as $\alpha$ decreases provided $\alpha<c x_1$. Also, we have $\sup_{\alpha< cx_1} s(\vh(\alpha)) = 1$ as the case in the proof of Proposition~\ref{prop:smoothness-control-linear}. One can check that $s(\vh(\alpha))$ is a continuous function for $\alpha< cx_1$ and thus it has range $[K_1/c, 1)$ by the mean value theorem.

Finally, we can establish the result: $K_1/c =  \sqrt{\frac{\sum_{x_i = \max \vx}d_i}{\sum^n_{j =1}d_j}}$ is the minimum of $s(\vh(\alpha))$ and $1$ is the maximum of $s(\vh(\alpha))$ occurring whenever $\alpha \geq cx_1 = \sqrt{\sum^n_{j =1}d_j}\max_i x_i$. Moreover, $s(\vh(\alpha))$ has a monotone property when $\alpha < \sqrt{\sum^n_{j =1}d_j}\max_i x_i$ and has range $\Big[ \sqrt{\frac{\sum_{x_i=\max \vx}d_i}{\sum^n_{j =1}d_j}},1\Big]$.

It is clear that the assumption on the ordering of the entries of $\vx$ will not affect this result.
\end{proof}

To prove Proposition~\ref{prop:smoothness-control-leaky-relu}, we first prove an analogous result for the identity function, that is, $\vh = \sigma(\vz) = \vz$.
\begin{proposition}\label{prop:smoothness-control-linear} Suppose $\vz_{\gM^\perp} \neq {\bf 0}$, then $s(\vz(\alpha))$ achieves its minimum $0$
if $\alpha = \langle \vz, \ve\rangle$. Moreover, $\sup_\alpha s(\vz(\alpha)) = 1$ where $s(\vz(\alpha))$ is close to $1$ when $\alpha$ is far away from $\langle \vz, \ve\rangle $.
\end{proposition}
Notice that Proposition~\ref{prop:smoothness-control-linear} does not consider the activation function.

\begin{proof}[Proof of Proposition~\ref{prop:smoothness-control-linear}]
We know that $0\leq s(\vz(\alpha))\leq 1$ and
$$
\begin{aligned}
s(\vz(\alpha)) &= 
\sqrt{1 - \frac{\|\vz_{\gM^\perp}\|^2}{\|\vz(\alpha)\|^2}}\\
&= \sqrt{1 - \frac{\|\vz_{\gM^\perp}\|^2}{\|\vz_{\gM^\perp}\|^2 + \|\vz(\alpha)_{\gM}\|^2}}\\
&=
 \sqrt{1 - \frac{\|\vz_{\gM^\perp}\|^2}{\|\vz_{\gM^\perp}\|^2 + \|\vz_{\gM}-\alpha\ve\|^2}}.
\end{aligned}
$$

Suppose  $s(\vz(\alpha)) = 1$. Then we have
$
\frac{\|\vz_{\gM^\perp}\|^2}{\|\vz_{\gM^\perp}\|^2 + \|\vz_{\gM}-\alpha\ve\|^2} = 0$
which forces $\|\vz_{\gM^\perp}\| = 0$. However, this contradicts the hypothesis $\vz_{\gM^\perp}\neq 0$. So $s(\vz(\alpha))$ cannot attain its maximum.

But for any $0\leq t<1$, one can see that $s(\vz(\alpha))=t$ if and only if
$$
\begin{aligned}
&\sqrt{1 - \frac{\|\vz_{\gM^\perp}\|^2}{\|\vz_{\gM^\perp}\|^2 + \|\vz_{\gM}-\alpha\ve\|^2}}  = t \\
\Leftrightarrow &\frac{\|\vz_{\gM^\perp}\|^2}{\|\vz_{\gM^\perp}\|^2 + \|\vz_{\gM}-\alpha\ve\|^2} = 1-t^2 \\
\Leftrightarrow & \|\vz_{\gM^\perp}\|^2 = (1-t^2)\big(\|\vz_{\gM^\perp}\|^2 + \|\vz_{\gM}-\alpha\ve\|^2\big)\\
\Leftrightarrow & t^2\|\vz_{\gM^\perp}\|^2 = (1-t^2) \|\vz_{\gM}-\alpha\ve\|^2\\
\Leftrightarrow & \|\vz_{\gM}-\alpha\ve\| = \sqrt{\frac{t^2}{1-t^2}} \cdot \|\vz_{\gM^\perp}\|
\end{aligned}
$$
This implies that $\sup_\alpha s(\vz(\alpha)) = 1$ and $s(\vz(\alpha))$ achieves its minimum $0$ if and only if $\alpha = \langle \vz, \ve\rangle$. It is clear that $s(\vz(\alpha))$ get closer to $1$ when $\alpha$ is going far away from $\langle \vz, \ve\rangle $. i.e., $|\alpha - \langle \vz, \ve\rangle| = \|\vz_{\gM}-\alpha\ve\|$ is increasing.
\end{proof}

\begin{proof}[Proof of Proposition~\ref{prop:smoothness-control-leaky-relu}]
First, we notice that leaky ReLU has the following two properties
\begin{enumerate}
    \item $\sigma_a(x)>0$ for $x\gg 0$ and $\sigma_a(x)<0$ for $x\ll 0$.
    \item $\sigma_a$ is a non-trivial linear map for $x\gg 0$. 
\end{enumerate}

We will use Property $1$ to show that $\min_\alpha s(\vh(\alpha)) = 0$ and  Property $2$ to show that $\sup_\alpha s(\vh(\alpha)) = 1$. Notice that $\sigma_a(x)<0$ for $x\ll 0$ implies that there exists a sufficient small $\alpha_2<0$ s.t. all of the entries of $\vh(\alpha_2)$ are negative and hence $|\langle \vh(\alpha_2), \ve\rangle|<0$. Similarly, $\sigma_a(x)>0$ for $x\gg 0$ implies that there exists a sufficient large $\alpha_1>0$ s.t. all of the entries of $\vh(\alpha_1)$ are positive and hence $|\langle \vh(\alpha_1), \ve\rangle|>0$. Since $|\langle \vh(\alpha), \ve\rangle|$ is a continuous function of $\alpha$ on $[\alpha_1, \alpha_2]$, the Intermediate Value Theorem follows that there exists an $\alpha\in (\alpha_1, \alpha_2)$ s.t. $|\langle \vh(\alpha), \ve\rangle|=0$. Thus by definition $s(\vh(\alpha)) = |\langle \vh(\alpha), \ve\rangle| / \|\vh(\alpha)\|$, we see that $\min_\alpha s(\vh(\alpha)) = 0$.

On the other hand, since $\sigma_a$ is a non-trivial linear map for $x\gg 0$, we may assume $\sigma_a(x) = cx$  for $x>x_0$ where $c\neq 0 $ is some non-zero constant and $x_0 >0$ is some positive constant.
Then we can choose an $\alpha_0>\langle \vz, \ve\rangle$ s.t. for any $\alpha\geq \alpha_0$, all of the entries of $\vz(\alpha)$ are greater than $x_0$.
Then whenever $\alpha\geq \alpha_0$, we have $\vh(\alpha) = \sigma_a(\vz(\alpha)) = c\vz(\alpha)$. This implies 
$$
\begin{aligned}
s(\vh(\alpha)) = \frac{|\langle \vh(\alpha), \ve\rangle|}{\|\vh(\alpha)\|} = \frac{|\langle c\vz(\alpha), \ve\rangle|}{\|c\vz(\alpha)\|} = \frac{|\langle \vz(\alpha), \ve\rangle|}{\|\vz(\alpha)\|} = s(\vz(\alpha)).
\end{aligned}
$$
Thus $\sup_\alpha s(\vh(\alpha)) = 1$ follows from the Proof of Proposition~\ref{prop:smoothness-control-linear} where we see that $\sup_\alpha s(\vz(\alpha)) = 1$ since $s(\vz(\alpha))$ gets closer to $1$ as $\alpha$ increases.

\end{proof}
\begin{remark}
 Indeed, it holds for any continuous function $f:\sR\to\sR$ satisfying the following 
\begin{enumerate}
    \item \label{condition:general smoothness min} $f(x)>0$ for $x\gg 0$, $f(x)<0$ for $x\ll 0$ or $f(x)<0$ for $x\gg 0$, $f(x)>0$ for $x\ll 0$,
    \item \label{condition:general smoothness sup}  
    $f$ is a  non-trivial linear map  for $x\gg 0$ or $x\ll 0$. 
\end{enumerate}
One can check the proof above only depends on these two properties. It is worth mentioning that most activation functions, e.g. leaky LU, SiLU, $\tanh$, satisfy condition~\ref{condition:general smoothness min}.
\end{remark}

\begin{proof}[Proof of Corollary~\ref{cor:smoothness effective control}]
For any $\alpha$, we notice that $\|\vz\|_{\gM^\perp} = \|\vz_{\gM^\perp}\|_F = \|\vz(\alpha)\|_{\gM^\perp}$ since $\alpha$ only changes the component of $\vz$ in the eigenspace $\gM$.
Also, Propositions~\ref{prop:relu-smoothness-geometric} and \ref{prop:leaky-relu-smoothness-geometric} show that $\|\vz(\alpha)\|_{\gM^\perp}\geq \|\vh(\alpha)\|_{\gM^\perp}$ whenever $\vh(\alpha) = \sigma(\vz(\alpha))$ or $\sigma_a(\vz(\alpha))$.
Therefore, we see that $\|\vz\|_{\gM^\perp}\geq \|\vh(\alpha)\|_{\gM^\perp}$ holds for any $\alpha$.
Since $\vz_{\gM^\perp}\neq 0$, $s(\vz)$ must lie in $[0,1)$.

\end{proof}

\section{Additional Experimental Results}\label{appendix-addition-results}
Table~\ref{tab:simple-t-test} shows the t-test results at 0.95 confidence (over 100 independent runs), which compares models with and without SCT on different benchmark node classification tasks; a larger value means the accuracy improvement due to using SCT is more statistically significant. 
\begin{table}[!ht]
\fontsize{9}{9}\selectfont\centering
\begin{tabular}{c|ccccc}
\specialrule{1.2pt}{1pt}{1pt}
     Layers
     &
     \hspace{0.7cm}2
     \hspace{0.7cm}
     &
     \hspace{0.7cm}4
     \hspace{0.7cm}
     \hspace{0.7cm}
     &\hspace{0.7cm}
     \hspace{0.7cm}16
     \hspace{0.7cm}
     &
     \hspace{0.7cm}32
     \hspace{0.7cm}
     \\
 \specialrule{1.2pt}{1pt}{1pt}
     \multicolumn{5}{c}{\textbf{Cora}}\\
     \hline
         GCN-SCT
         &
         $13.77$
         &
         $8.72$
         &
         $6.40$
         &
         $5.95$
         \\
         GCNII-SCT
         &
         $16.59$         
         &
         $5.85$ 
         &
         $0.49$
         &
         $0.38$
         \\
         EGNN-SCT
         &
         $7.29$         
         &
         $3.04$         
         &
         $0.15$
         &
         $-1.16$
         \\
\specialrule{1.2pt}{1pt}{1pt}
     \multicolumn{5}{c}{\textbf{Citeseer}}\\
 \hline
         GCN-SCT
         &
         $0.58$
         &
         $18.44$
         &
         $83.92$
         &
         $16.58$
         \\
         GCNII-SCT
         &
         $26.28$
         &
         $7.93$
         &
         $-1.07$
         &
         $0.39$
         \\
         EGNN-SCT
         &
         $15.14$
         &
         $0.13$
         &
         $0.46$
         &
         $1.55$
         \\
\specialrule{1.2pt}{1pt}{1pt}
     \multicolumn{5}{c}{\textbf{PubMed}}\\
 \hline
         GCN-SCT
         &
         {$-0.77$}
         &
         $1.84$
         &
         $17.03$
         &
         $22.99$
         \\
         GCNII-SCT
         &
         $1.89$
         &
         $0.21$
         &
         $9.51$
         &
         $8.64$
         \\
         EGNN-SCT
         &
         $1.57$
         &
         $4.26$
         &
         $4.30$
         &
         $2.68$
         \\
    \specialrule{1.2pt}{1pt}{1pt}
    \end{tabular}
\caption{
We conduct t-test 
at 0.95 confidence to compare models with and without SCT on different benchmark node classification tasks, where 
$\text{t-score} = ({\mu_{\text{*-{SCT}}} - \mu_{\text{*}}})/
    {\sqrt{{\sigma^2_{\text{*-SCT}}}/{n} + {\sigma^2_{\text{*}}}/{n}}}$.
 We observe that in general SCT provides significant improvements - only fails to improve in very few cases and by a marginal amount. A larger t-score means a more significant improvement.
 }
    \label{tab:simple-t-test}
\end{table}

\section{Experimental Details}\label{appendix-exp-details}
This part includes the missing details about experimental configurations and additional experimental results for Section~\ref{sec:experiments}. All tasks we run using Nvidia RTX 3090, GV100, and Tesla T4 GPUs. All computational performance metrics, including timing procedures, are run using Tesla T4 GPUs from Google Colab.

\subsection{Dataset details}\label{appendix:datasets}
In this section, we briefly describe the benchmark datasets used. Table~\ref{appendix:dataset:statistics} provides additional details about the underlying graph representation.

\medskip
\textbf{Citation Datasets:} 
The five citation datasets considered are Cora, Citeseer PubMed, Coauthor-Physics, and Ogbn-arxiv. Each dataset is represented by a graph with nodes representing academic publications, features encoding a bag-of-words description, labels classifying the publication type, and edges representing citations.

\medskip
\textbf{Web Knowledge-Base Datasets:}
The three web knowledge-base datasets are Cornell, Texas, and Wisconsin.
Each dataset is represented by a graph with nodes representing CS department webpages, features encoding a bag-of-words description, edges representing hyper-link connections, and labels classifying the webpage type.

\medskip
\textbf{Wikipedia Network Datasets:} 
The two Wikipedia network datasets are Chameleon and Squirrel.
Each dataset is represented by a graph with nodes representing CS department webpages, features encoding a bag-of-words description, edges representing hyper-link connections, and labels classifying the webpage type.

\begin{table}[!ht]
\fontsize{8.5}{8.5}\selectfont
\centering
\begin{tabular}{ccccccc}
\specialrule{1.2pt}{1pt}{1pt}
    & \textbf{\# Nodes} & \textbf{\# Edges} & \textbf{\# Features} & \textbf{\# Classes} & Splits (Train/Val/Test) \\
 \hline
 Cornell &  $183$ & $295$ & $1,703$ & $5$ & 48/32/20\%\\
 Texas & $181$ & $309$ & $1,703$ & $5$ & 48/32/20\%\\
 Wisconsin & $251$ & $499$ & $1,703$ & $5$ & 48/32/20\%\\
 Chameleon & $2,277$ & $36,101$ & $2,325$ & $5$ & 48/32/20\%\\
 Squirrel & $5,201$ & $217,073$ & $2,089$ & $5$ & 48/32/20\%\\
 Citeseer & $3,727$ & $4,732$ & $3,703$ & $6$ & 
 120/500/1000 
 \\
 Cora & $2,708$ & $5,429$ & $1,433$ & $7$ & 
 140/500/1000 
 \\
 PubMed & $19,717$ & $44,338$ & $500$ & $3$ & 
 60/500/1000
 \\
 Coauthor-Physics &  34,493 & 247,962 & 8415 & 5 & 100/150/34,243\\
 Ogbn-arxiv       & 169,343 & 1,166,243 & 128 & 40 & 90,941/29,799/48,603\\
    \specialrule{1.2pt}{1pt}{1pt}
    \end{tabular}
    \caption{
    Graph statistics.
    }
    \label{appendix:dataset:statistics}
\end{table}

\subsection{{Model size and computational time for citation datasets}}
Table~\ref{table:model_sizes} compares the model size and computational time for experiments on citation datasets in Section~\ref{subsec-node-classification}.

\begin{table}[!ht]
\fontsize{8.5}{8.5}\selectfont
\centering
\begin{tabular}{c|ccc}
\specialrule{1.2pt}{1pt}{1pt}
     
     &
     \# Parameters
     &
     Training Time (s)
     &
     Inference Time (ms)
     \\
 \specialrule{1.2pt}{1pt}{1pt}
     \multicolumn{4}{c}{\textbf{Cora}}\\
     \hline
         GCN
         &
         100,423
         &
         8.4
         &
         1.6
         \\
         GCNII
         &
         110,535
         &
         10.0
         &
         2.1
         \\
         GCNII
         &
         708,743
         &
         57.6
         &
         12.3
         \\
         GCNII-SCT
         &
         1,237,127
         &
         110.3
         &
         29.6
         \\
         EGNN
         &
         712,839
         &
         65.6
         &
         14.4
         \\
         EGNN-SCT
         &
         316,551
         &
         24.8
         &
         4.5
         \\
\specialrule{1.2pt}{1pt}{1pt}
     \multicolumn{4}{c}{\textbf{Citeseer}}\\
 \hline
         GCN
         &
         245,638
         &
         8.3
         &
         1.5
         \\
         GCN-SCT
         &
         301,830
         &
         15.5
         &
         4.0
         \\
         GCNII
         &
         999,174
         &
         57.6
         &
         12.3
         \\
         GCNII-SCT
         &
         1,001,222
         &
         65.9
         &
         15.7
         \\
         EGNN
         &
         739,078
         &
         39.6
         &
         7.2
         \\
         EGNN-SCT
         &
         540,934
         &
         24.0
         &
         5.8
         \\
\specialrule{1.2pt}{1pt}{1pt}
     \multicolumn{4}{c}{\textbf{PubMed}}\\
 \hline
         GCN
         &
         40,451
         &
         9.0
         &
         1.8
         \\
         GCN-SCT
         &
         40,707
         &
         11.1
         &
         2.2
         \\
         GCNII
         &
         326,659
         &
         98.2
         &
         12.8
         \\
         GCNII-SCT
         &
         590,851
         &
         71.7
         &
         17.4
         \\
         EGNN
         &
         592,899
         &
         93.7
         &
         2.5
         \\
         EGNN-SCT
         &
         130,563
         &
         16.0
         &
         3.1
         \\
\specialrule{1.2pt}{1pt}{1pt}
     \multicolumn{4}{c}{\textbf{Coauthor-Physics}}\\
 \hline
         GCN
         &
         547,141
         &
         35.2
         &
         8.0
         \\
         GCN-SCT
         &
         547,397
         &
         33.9
         &
         8.3
         \\
         GCNII
         &
         555,333
         &
         49.1
         &
         10.3
         \\
         GCNII-SCT
         &
         555,461
         &
         67.0
         &
         9.5
         \\
         EGNN
         &
         672,069
         &
         176.4
         &
         47.9
         \\
         EGNN-SCT
         &
         572,229
         &
         51.7
         &
         14.8
         \\
\specialrule{1.2pt}{1pt}{1pt}
     \multicolumn{4}{c}{\textbf{Ogbn-arxiv}}\\
 \hline
         GCN
         &
         27,240
         &
         50.4
         &
         21.1
         \\
         GCN-SCT
         &
         28,392
         &
         62.6
         &
         24.4
         \\
         GCNII
         &
         76,392
         &
         205.4
         &
         94.8
         \\
         GCNII-SCT
         &
         80,616
         &
         253.0
         &
         108.9
         \\
         EGNN
         &
         77,416
         &
         206.8
         &
         98.0
         \\
         EGNN-SCT
         &
         81,640
         &
         254.0
         &
         112.3
         \\
    \specialrule{1.2pt}{1pt}{1pt}
    \end{tabular}
\caption{ 
Number of model parameters for varying numbers of layers using the optimal model hyperparameters. The SCT is added at each layer and the size of the additional parameters scales with the number of eigenvectors with an eigenvalue of one for matrix $\mG$ in \eqref{eq:GCN}. 
}
    \label{table:model_sizes}
\end{table}

\subsection{Additional Section~\ref{subsec-node-classification}
details for citation datasets}\label{appendix:semi-sup}
Table~\ref{table:hyperparameter:gridsearch} lists the hyperparameters used in the grid search in generating the results in Table~\ref{table:node:variable-layers}. Also, Table~\ref{table:node:variable-relu-leakyrelu} reports the classification accuracy of different models with different depths using either ReLU or leaky ReLU. 
\begin{table}[!ht]
\fontsize{8.5}{8.5}\selectfont
\centering
\begin{tabular}{cc}
\specialrule{1.2pt}{1pt}{1pt}
\hspace{0.5cm}    \textbf{Parameter} \hspace{0.5cm}  &\hspace{1cm}  \textbf{Values} \hspace{1cm} \\
 \hline
Learning Rate &  $\{1e\text{-}4, 1e\text{-}3, 1e\text{-}2 \}$ \\
Weight Decay (FC) &  $\{0, 1e\text{-}4, 5e\text{-}4, 1e\text{-}3, 5e\text{-}3, 1e\text{-}2 \}$ \\
Weight Decay (Conv) &  $\{0, 1e\text{-}4, 5e\text{-}4, 1e\text{-}3, 5e\text{-}3, 1e\text{-}2 \}$ \\
Dropout &  $\{0.1, 0.2, 0.3, 0.4, 0.5, 0.6, 0.7, 0.8, 0.9\}$ \\
Hidden Channels &  $\{16, 32, 64, 128\}$ \\
\hline
GCNII-$\alpha$ &  $\{0.1, 0.2, 0.3, 0.4, 0.5, 0.6, 0.7, 0.8, 0.9\}$ \\
GCNII-$\theta$ &  $\{0.1, 0.2, 0.3, 0.4, 0.5, 0.6, 0.7, 0.8, 0.9\}$ \\
\hline
EGNN-$c_\text{max}$ &  $\{0.5, 1.0, 1.5, 2.0\}$ \\
EGNN-$\alpha$ &  $\{0.1, 0.2, 0.3, 0.4, 0.5, 0.6, 0.7, 0.8, 0.9\}$ \\
EGNN-$\theta$ &  $\{0.1, 0.2, 0.3, 0.4, 0.5, 0.6, 0.7, 0.8, 0.9\}$ \\
    \specialrule{1.2pt}{1pt}{1pt}
    \end{tabular}
    \caption{
    Hyperparameter grid search for Table~\ref{table:node:variable-layers}.
    }
    \label{table:hyperparameter:gridsearch}
\end{table}

\begin{table}[!ht]
\fontsize{7.5}{7.5}\selectfont
\centering
\begin{tabular}{c|ccccc}
\specialrule{1.2pt}{1pt}{1pt}
     Layers & \hspace{0.5cm}\ \ \ \ \ \ \ 
     \hspace{-0.1cm}2\hspace{-0.1cm} \hspace{0.5cm}\ \ \ \ \ \ \ 
     &\hspace{0.5cm}\ \ \ \ \ \ \ 
     \hspace{-0.1cm}4\hspace{-0.1cm}\hspace{0.5cm}\ \ \ \ \ \ \  
     \hspace{-0.1cm}&\hspace{-0.1cm}\hspace{0.5cm}\ \ \ \ \ \ \ 
     \hspace{-0.1cm}16\hspace{-0.1cm}\hspace{0.5cm}\ \ \ \ \ \ \ 
     &\hspace{0.5cm}\ \ \ \ \ \ \  
     \hspace{-0.1cm}32\hspace{-0.1cm} \hspace{0.5cm}\ \ \ \ \ \ \  
     \\
 \specialrule{1.2pt}{1pt}{1pt}
     \multicolumn{5}{c}{\textbf{Cora}}\\
     \hline
EGNN/EGNN-SCT & $83.2$/$\mathbf{83.4}$ & {$84.2$/$\mathbf{84.3}$}
& {{${85.4}$}/$\mathbf{85.5}$} & {$85.3$/$\mathbf{85.5}$} \\
\specialrule{1.2pt}{1pt}{1pt}
     \multicolumn{5}{c}{\textbf{Citeseer}}\\
 \hline
EGNN/EGNN-SCT & $72.0$/$\mathbf{72.1}$ & $71.9$/$\mathbf{72.3}$ 
& $72.4$/$\mathbf{72.6}$  & $72.3$/$\mathbf{72.8}$\\ 
\specialrule{1.2pt}{1pt}{1pt}
\multicolumn{5}{c}{\textbf{PubMed}}\\
 \hline
EGNN/EGNN-SCT & $79.2$/$\mathbf{79.4}$ & $79.5$/$\mathbf{79.8}$  & $\mathbf{80.1}$/$\mathbf{80.1}$ & $80.0$/$\mathbf{80.2}$
         \\ 
    \specialrule{1.2pt}{1pt}{1pt}
     \multicolumn{5}{c}{\textbf{Coauthor-Physics}}\\
 \hline
EGNN/EGNN-SCT & $92.6$/$\mathbf{92.8}$ & $92.9$/$\mathbf{93.0}$ 
         & $93.1$/$\mathbf{93.3}$  &  $\mathbf{93.3}$/$\mathbf{93.3}$    
         \\
    \specialrule{1.2pt}{1pt}{1pt}
     \multicolumn{5}{c}{\textbf{Ogbn-arxiv}}\\
 \hline
EGNN/EGNN-SCT &  $68.4$/$\mathbf{68.5}$ & $71.1$/{$\mathbf{71.3}$} 
         & $72.7$/$\mathbf{73.0}$ &  $72.7$/$\mathbf{72.9}$ \\
    \specialrule{1.2pt}{1pt}{1pt}
    \end{tabular}
\caption{
Test accuracy for EGNN and EGNN-SCT using SReLU activation function of varying depth on citation networks with the split discussed in Section~\ref{subsec-node-classification}. 
(Unit:\%)
 } 
    \label{table:node:variable-layers-appendix}
\end{table}

\subsubsection{Vanishing gradients}
{Figure~\ref{fig:vanishing-gradients} shows the vanishing gradient problem for training deep GCN -- with or without SCT -- in comparison to models like GCNII and EGNN. This figure plots $||\partial\mH^{\text{out}}/\partial\mH^l||$ for layers $l\in[0,32]$ as the training epochs run from $0$ to $100$. 
Figures~\ref{fig:vanishing-gradients} (a) and (b) illustrate the vanishing gradient issue for GCN and that it persists for GCN-SCT. Figures~\ref{fig:vanishing-gradients} (c) and (e) illustrate that GCNII and EGNN do not suffer from vanishing gradients, and furthermore, because these models connect $\mH^0$ to every layer, the gradient with respect to the weights in the first layer is nonzero.
What is interesting about the addition of SCT to both EGNN and GCNII is that the intermediate gradients become large as the training epochs progress shown in Figure~\ref{fig:vanishing-gradients} (d) and (f). 
}

\begin{figure}[!ht]
\begin{center}
\begin{tabular}{cc}
\includegraphics[width=0.45\linewidth]{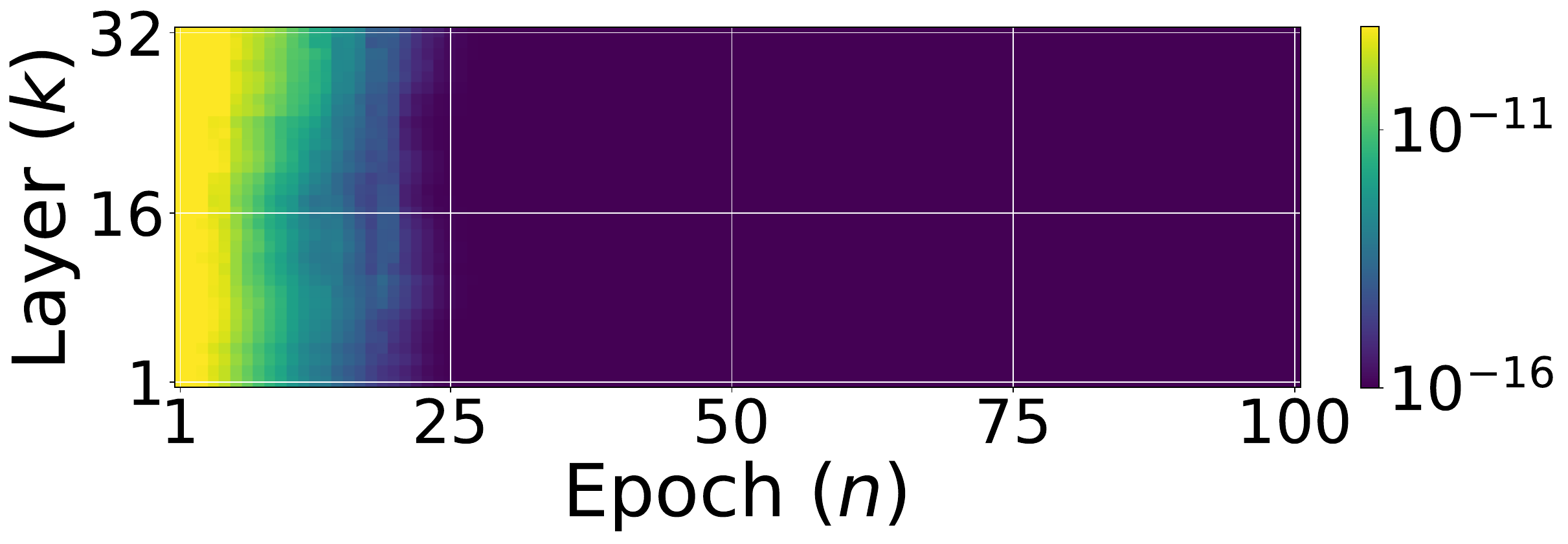}&
\includegraphics[width=0.45\linewidth]{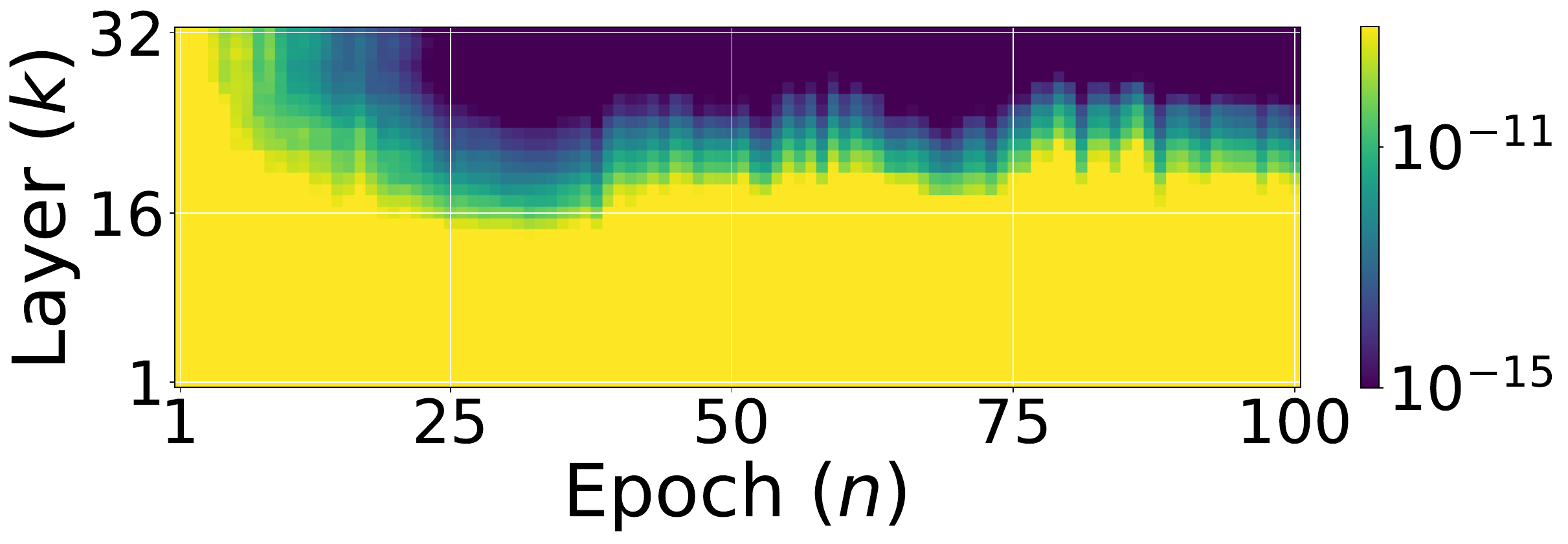}\\
{\small (a) GCN} & {\small (b) GCN-SCT}\\
\includegraphics[width=0.45\linewidth]{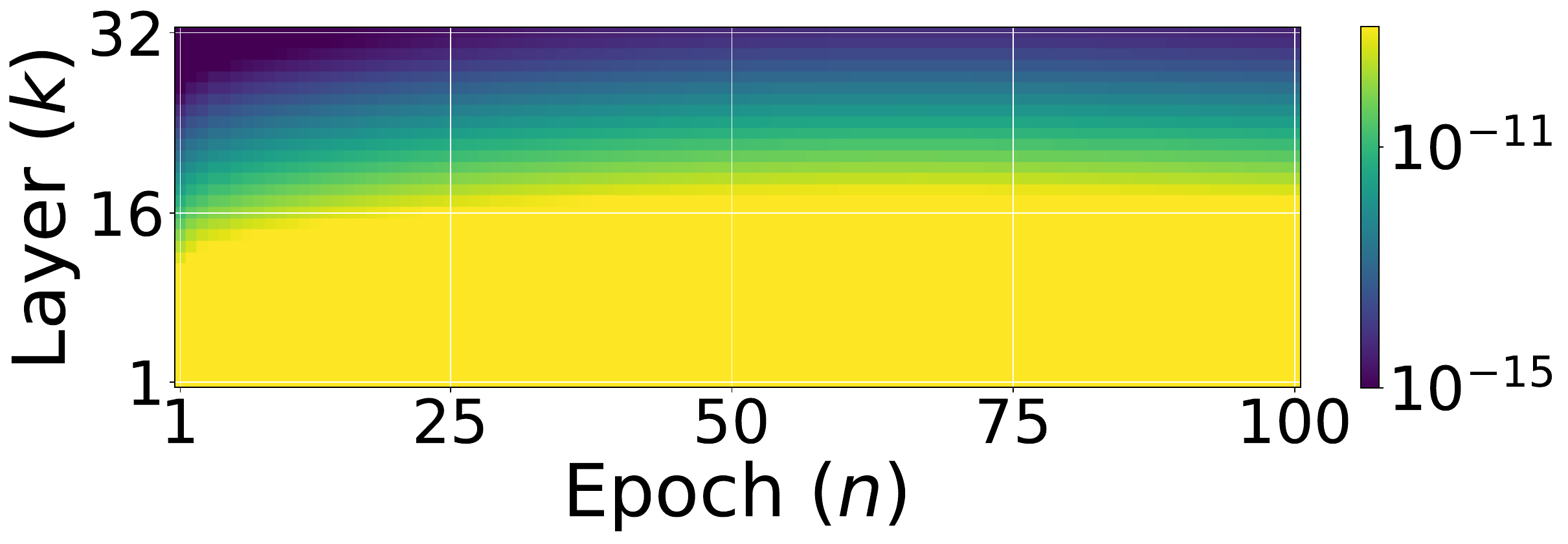}&
\includegraphics[width=0.45\linewidth]{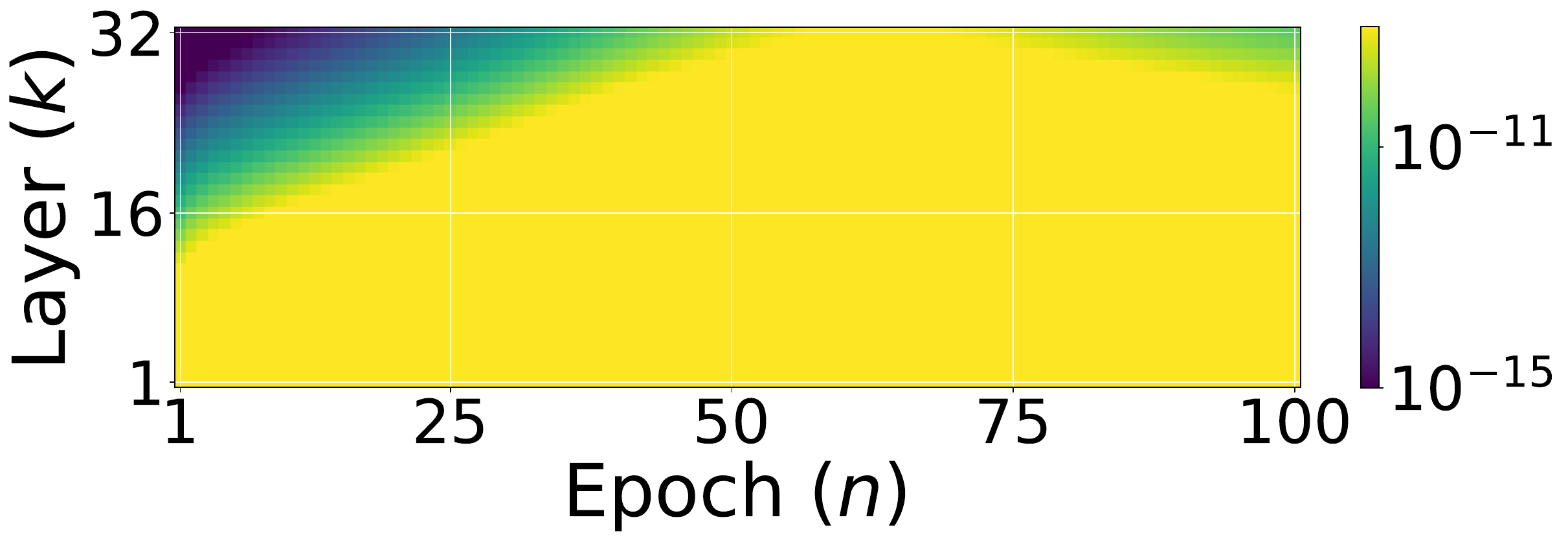}\\
{\small (c) GCNII} & {\small (d) GCNII-SCT} \\
\includegraphics[width=0.45\linewidth]{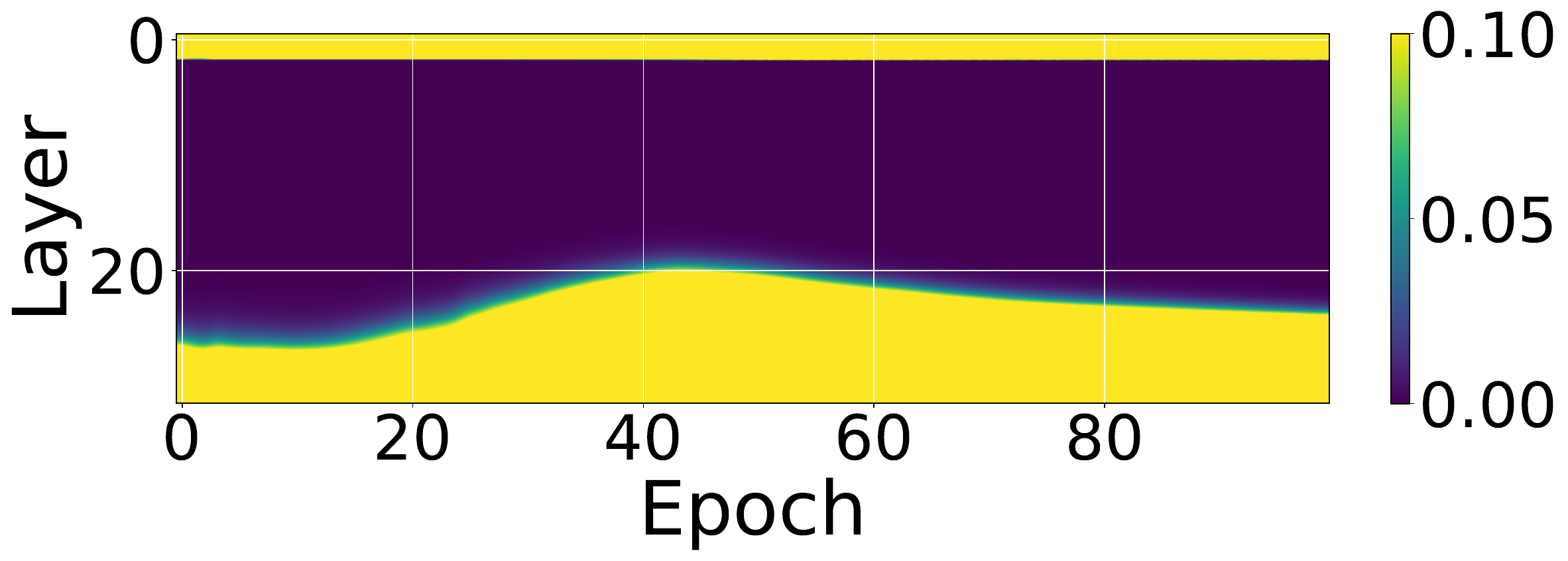} &
\includegraphics[width=0.45\linewidth]{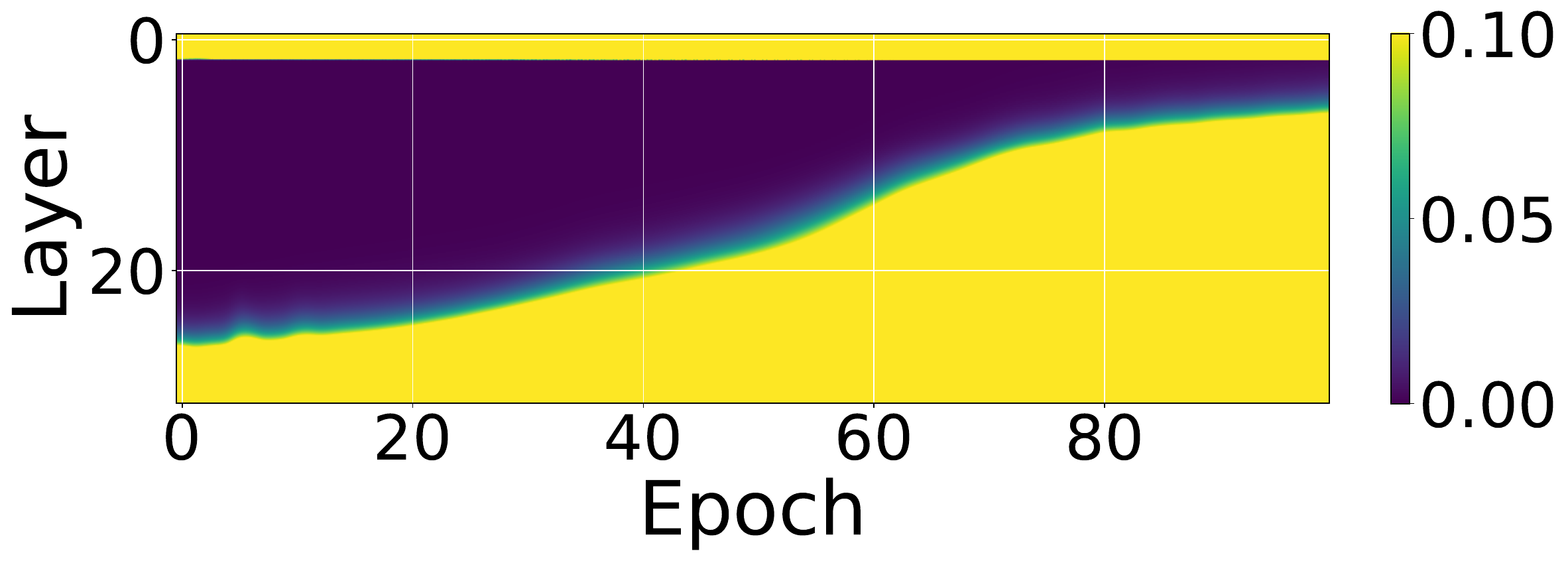}\\
{\small (e) EGNN} & {\small (f) EGNN-SCT}\\
\end{tabular}
  \end{center}
  \caption{
  Training gradients for $||\partial \mH^\text{out}/\partial \mH^l||$ for $l\in[0,32]$ layers and 100 training epochs on the Citeseer dataset. Here, all models have 32 layers and 16 hidden dimensions for each layer. We observe that (a) GCN suffers from vanishing gradients. By contrast (c) GCNII and (e) EGNN do not suffer from vanishing gradients, and we can observe their skip connection to $\mH^0$. Because these models (GCNII/GCNII-SCT and EGNN/EGNN-SCT) connect $\mH^0$ to every layer, the gradient at the first layer is nonzero. We notice that while SCT does not overcome vanishing gradients for (b) GCN-SCT, it is able to increase the norm of the gradients for the intermediate layers in (d) GCNII-SCT and (f) EGNN-SCT. 
  }\label{fig:vanishing-gradients}
\end{figure}

\begin{table}[!ht]
\fontsize{4.0}{4.0}\selectfont
\centering
\begin{tabular}{c|cccc|cccc}
\specialrule{1.2pt}{1pt}{1pt}
\multicolumn{9}{c}{Cora}\\ \hline
 & \multicolumn{4}{c}{ReLU} & \multicolumn{4}{c}{leaky ReLU}\\
\hline
\ \ \ \ \ Layers\ \ \ \ \  &\ \ \ \ \  2\ \ \ \ \  &\ \ \ \ \  4\ \ \ \ \  &\ \ \ \ \  16\ \ \ \ \  &\ \ \ \ \  32\ \ \ \ \   &\ \ \ \ \  2\ \ \ \ \  &\ \ \ \ \  4\ \ \ \ \  &\ \ \ \ \  16\ \ \ \ \  &\ \ \ \ \  32\ \ \ \ \ \\
\hline
GCN-SCT & $81.2$ & $80.3$ & $71.4$ & $67.2$ & $82.9$ & $82.8$ & $68.0$ & $65.5$  \\
GCNII-SCT & $83.5$ & $83.8$ & $82.7$ & $83.3$ & $83.8$ & $84.8$ & $84.8$ & $85.5$ \\
EGNN-SCT & $84.1$ & $83.8$ & $82.3$ & $80.8$ & $83.7$ & $84.5$ & $83.3$ & $82.0$ \\
\specialrule{1.2pt}{1pt}{1pt}
\multicolumn{9}{c}{Citeseer}\\ \hline
 & \multicolumn{4}{c}{ReLU} & \multicolumn{4}{c}{leaky ReLU}\\
\hline
Layers & 2 & 4 & 16 & 32  & 2 & 4 & 16 & 32\\
\hline
GCN-SCT & 69.0 & 67.3 & 51.5 & 50.3 & 69.9 & 67.7 & 55.4 & 51.0  \\
GCNII-SCT & $72.8$ & $72.8$ & $72.8$ & $73.3$ & $72.8$ & $72.9$ & $73.8$ & $72.7$  \\
EGNN-SCT & $72.5$ & $72.0$  & $70.2$ & $71.8$ & $73.1$ & $71.7$ & $72.6$ & $72.9$  \\
\specialrule{1.2pt}{1pt}{1pt}
\multicolumn{9}{c}{PubMed}\\ \hline
 & \multicolumn{4}{c}{ReLU} & \multicolumn{4}{c}{leaky ReLU}\\
\hline
Layers & 2 & 4 & 16 & 32  & 2 & 4 & 16 & 32\\
\hline
GCN-SCT & 79.4 & 78.2 & 75.9 & 77.0 & 79.8 & 78.4 & 76.1 & 76.9  \\
GCNII-SCT & $79.7$ & $80.1$ & $80.7$ & $80.7$ & $79.6$ & $80.0$ & $80.3$ & $80.7$ \\
EGNN-SCT & $79.7$ & $80.1$ & $80.0$ & $80.4$ &  $79.8$ & $80.4$ & $80.3$ & $80.2$  \\
\specialrule{1.2pt}{1pt}{1pt}
\multicolumn{9}{c}{Coauthor-Physics}\\ \hline
 & \multicolumn{4}{c}{ReLU} & \multicolumn{4}{c}{leaky ReLU}\\
\hline
Layers & 2 & 4 & 16 & 32  & 2 & 4 & 16 & 32\\
\hline
GCN-SCT & $91.8\pm1.6$ & $91.6\pm3.0$ & $44.5\pm13.0$  & $42.6\pm17.0$  & $92.6\pm1.6$ & $92.5\pm5.9$ & $50.9\pm15.0$ & $43.6\pm16.0$  \\
GCNII-SCT & $94.4\pm0.4$ & $93.5\pm1.2$ & $93.7\pm0.7$ & $93.8\pm0.6$ & $94.0\pm0.4$ & $94.2\pm0.3$ & $93.3\pm0.7$ & $94.1\pm0.3$  \\
EGNN-SCT & $93.6\pm0.7$ & $94.1\pm0.4$ & $93.4\pm0.8$ & $93.8\pm1.3$ & $93.9\pm0.7$ & $94.0\pm0.7$ & $94.0\pm0.7$ & $93.3\pm0.9$  \\
\specialrule{1.2pt}{1pt}{1pt}
\multicolumn{9}{c}{Ogbn-arxiv}\\ \hline
 & \multicolumn{4}{c}{ReLU} & \multicolumn{4}{c}{leaky ReLU}\\
\hline
Layers & 2 & 4 & 16 & 32  & 2 & 4 & 16 & 32\\
\hline
GCN-SCT & $71.7\pm0.3$ &  $72.6\pm0.3$ & $71.4\pm0.2$ & $71.9\pm0.3$  & $72.1\pm0.3$ & $72.7\pm0.3$ & $72.3\pm0.2$ & $72.3\pm0.3$  \\
GCNII-SCT & $71.4\pm0.3$ & $72.1\pm0.3$ & $72.2\pm0.2$ & $71.8\pm0.2$  & $72.0\pm0.3$ & $72.2\pm0.2$ & $72.4\pm0.3$ & $72.1\pm0.3$ \\
EGNN-SCT & $68.5\pm0.6$ & $71.0\pm0.5$ & $72.8\pm0.5$ & $72.1\pm0.6$  & $67.7\pm0.5$ & $71.3\pm0.5$ & $72.3\pm0.5$ & $72.3\pm0.5$  \\
\specialrule{1.2pt}{1pt}{1pt}
\end{tabular}
\caption{
Test accuracy results for models of varying depth with ReLU or leaky ReLU activation function on the citation network datasets using the split discussed in Section~\ref{subsec-node-classification}. 
}\label{table:node:variable-relu-leakyrelu}
\end{table}

\subsection{Additional Section~\ref{subsec-node-classification} 
details for other datasets}\label{appendix:full-sup}
{Table~\ref{table:node:web-wiki-mean-std-fixed-layers} reports the mean test accuracy and standard deviation over ten folds of the WebKB and WikipediaNetwork datasets using SCT-based models.}

Table~\ref{table:node:web-wiki-time-fixed-layers} lists the average computational time for each epoch for different models of the same depth -- 8 layers. These results show that integrating SCT into GNNs only results in a small amount of computational overhead.

\begin{table}[!ht]
\fontsize{8.5}{8.5}\selectfont
\centering
{
\begin{tabular}{cccccc}
\specialrule{1.2pt}{1pt}{1pt}
    & \textbf{Cornell} & \textbf{Texas} & \textbf{Wisconsin} & \textbf{Chameleon} & \textbf{Squirrel} \\
 \hline
 
GCN-SCT & $55.95\pm8.5$ & $62.16\pm5.7$ & $54.71\pm4.4$ & $38.44\pm4.3$ & $35.31\pm1.9$ \\
GCNII-SCT & $75.41\pm2.2$ & $83.34\pm4.5$ & $86.08\pm3.8$ & $64.52\pm2.2$ & $47.51\pm1.4$ \\
    \specialrule{1.2pt}{1pt}{1pt}
    \end{tabular}
}
    \caption{
    Test mean $\pm$ standard deviation accuracy from $10$ fold cross-validation on five heterophilic datasets with fixed $48/32/20\%$ splits. The depth of each model is 8 layers {with 16 hidden channels}.
    (Unit: second)
    }
    \label{table:node:web-wiki-mean-std-fixed-layers}
\end{table}

\begin{table}[!ht]
\fontsize{8.5}{8.5}\selectfont
\centering
\begin{tabular}{cccccc}
\specialrule{1.2pt}{1pt}{1pt}
    & \textbf{Cornell} & \textbf{Texas} & \textbf{Wisconsin} & \textbf{Chameleon} & \textbf{Squirrel} \\
 \hline
GCN~\cite{kipf2017semisupervised} & $0.011$ & $0.013$ & $0.012$ & $0.011$ & $0.022$ \\
GCNII~\cite{chen2020simple} & $0.017$ & $0.018$ & $0.017$ & $0.013$ & $0.022$ \\
\hline
GCN-SCT & $0.015$ & $0.017$ & $0.015$ & $0.011$ & $0.023$ \\
GCNII-SCT & $0.017$ & $0.018$ & $0.017$ & $0.020$ & $0.025$ \\
    \specialrule{1.2pt}{1pt}{1pt}
    \end{tabular}
    \caption{
    Average computational time per epoch for five heterophilic datasets with fixed $48/32/20\%$ splits. The depth of each model is 8 layers {with 16 hidden channels}.
    (Unit: second)
    }
    \label{table:node:web-wiki-time-fixed-layers}
\end{table}

\clearpage

\end{document}